\documentclass[12pt]{article}
\usepackage[final]{nips_2017}

\usepackage[utf8]{inputenc} 
\usepackage[T1]{fontenc}    
\usepackage{hyperref}       
\usepackage{url}            
\usepackage{booktabs}       
\usepackage{amsfonts}       
\usepackage{nicefrac}       
\usepackage{microtype}      

\usepackage{amsmath}
\usepackage{amssymb}
\usepackage{amsthm}
\usepackage{mathrsfs}
\usepackage{mathtools}
\usepackage{dsfont}
\usepackage{enumerate}
\usepackage{pgf}
\usepackage{tikz}
\usetikzlibrary{arrows,automata}
\usepackage{framed}
\usepackage[,slantedGreek]{mathptmx}
\usepackage{stackrel}
\usepackage{subfig}

\usepackage{algorithm, algorithmic}
\floatname{algorithm}{AGB algorithm}

\usepackage{natbib}
\bibpunct{(}{)}{;}{a}{,}{,}
\usepackage{hyperref}
\hypersetup{
colorlinks = true,
urlcolor = blue,
linkcolor = blue,
citecolor = blue,
}

\newtheorem{rem}{Remark}[section]

\newtheorem{lem}{Lemma}[section]

\newtheorem{theo}{Theorem}[section]

\DeclarePairedDelimiterX{\infdivx}[2]{(}{)}{%
  #1\;\delimsize\|\;#2%
}

\DeclarePairedDelimiterX{\infdivy}[2]{(}{)}{%
  #1,#2%
}

\newcommand{\KL}{D_{{\rm \scriptsize{KL}}}\infdivx}
\newcommand{\JS}{D_{{\rm \scriptsize{JS}}}\infdivy}

\title{Some Theoretical Properties of GANs}

\author{
  G.~Biau\\
  Sorbonne Universit\'e, CNRS, LPSM\\
  Paris, France \\
  \texttt{gerard.biau@upmc.fr} \\
 \And
  B.~Cadre\\
  Univ Rennes, CNRS, IRMAR\\
  Rennes, France\\
   \texttt{benoit.cadre@ens-rennes.fr} \\
\And
  M.~Sangnier\\
  Sorbonne Universit\'e, CNRS, LPSM\\
  Paris, France \\
  \texttt{maxime.sangnier@upmc.fr} \\
  \And
  U.~Tanielian\\
  Sorbonne Universit\'e, CNRS, LPSM, Criteo\\
  Paris, France \\
  \texttt{u.tanielian@criteo.com}\\
}

\begin{document}

\maketitle

\begin{abstract}
Generative Adversarial Networks (GANs) are a class of generative algorithms that have been shown to produce state-of-the art samples, especially in the domain of image creation. The fundamental principle of GANs is to approximate the unknown distribution of a given data set by optimizing an objective function through an adversarial game between a family of generators and a family of discriminators. In this paper, we offer a better theoretical understanding of GANs by analyzing some of their mathematical and statistical properties. We study the deep connection between the adversarial principle underlying GANs and the Jensen-Shannon divergence, together with some optimality characteristics of the problem. An analysis of the role of the discriminator family via approximation arguments is also provided. In addition, taking a statistical point of view, we study the large sample properties of the estimated  distribution and prove in particular a central limit theorem. Some of our results are illustrated with simulated examples.
\end{abstract}

\section{Introduction}
The fields of machine learning and artificial intelligence have seen spectacular advances in recent years, one of the most promising being perhaps the success of Generative Adversarial Networks (GANs), introduced by \citet{GoPoMiXuWaOzCoBe14}. GANs are a class of generative algorithms implemented by a system of two neural networks contesting with each other in a zero-sum game framework. This technique is now recognized as being capable of generating photographs that look authentic to human observers \citep[e.g.,][]{SaGoZaChRaCg16}, and its spectrum of applications is growing at a fast pace, with impressive results in the domains of inpainting, speech, and 3D modeling, to name but a few. A survey of the most recent advances is given by \cite{Go16}.

The objective of GANs is to generate fake observations of a target distribution $p^{\star}$ from which only a true sample (e.g., real-life images represented using raw pixels) is available. It should be pointed out at the outset that the data involved in the domain are usually so complex that no exhaustive description of $p^{\star}$ by a classical parametric model is appropriate, nor its estimation by a traditional maximum likelihood approach. Similarly, the dimension of the samples is often very large, and this effectively excludes a strategy based on nonparametric density estimation techniques such as kernel or nearest neighbor smoothing, for example. In order to generate according to $p^{\star}$, GANs proceed by an adversarial scheme involving two components: a family of generators and a family of discriminators, which are both implemented by neural networks. The generators admit low-dimensional random observations with a known distribution (typically Gaussian or uniform) as input, and attempt to transform them into fake data that can match the distribution $p^{\star}$; on the other hand, the discriminators aim to accurately discriminate between the true observations from $p^{\star}$ and those produced by the generators. The generators and the discriminators are calibrated by optimizing an objective function in such a way that the distribution of the generated sample is as indistinguishable as possible from that of the original data. In pictorial terms, this process is often compared to a game of cops and robbers, in which a team of counterfeiters illegally produces banknotes and tries to make them undetectable in the eyes of a team of police officers, whose objective is of course the opposite. The competition pushes both teams to improve their methods until counterfeit money becomes indistinguishable (or not) from genuine currency.

From a mathematical point of view, here is how the generative process of GANs can be represented. All the densities that we consider in the article are supposed to be dominated by a fixed, known, measure $\mu$ on $E$, where $E$ is a Borel subset of $\mathds R^d$. This dominating measure is typically the Lebesgue or the counting measure, but, depending on the practical context, it can be a more complex measure. We assume to have at hand an i.i.d.~sample $X_1, \hdots, X_n$, drawn according to some unknown density $p^{\star}$ on $E$. These random variables model the available data, such as images or video sequences; they typically take their values in a high-dimensional space, so that the ambient dimension $d$ must be thought of as large. The generators as a whole have the form of a parametric family of functions from $\mathds R^{d'}$ to $E$ ($d' \ll d$), say $\mathscr G=\{G_{\theta}\}_{\theta \in \Theta}$, $\Theta \subset \mathds R^p$. Each function $G_{\theta}$ is intended to be applied to a $d'$-dimensional random variable $Z$ (sometimes called the noise---in most cases Gaussian or uniform), so that there is a natural family of densities associated with the generators, say $\mathscr P=\{p_{\theta}\}_{\theta \in \Theta}$, where, by definition, $G_{\theta}(Z)\stackrel{\mathscr L}{=}p_{\theta}{\rm d}\mu$. In this model, each density $p_{\theta}$ is a potential candidate to represent $p^{\star}$. On the other hand, the discriminators are described by a family of Borel functions from $E$ to $[0,1]$, say $\mathscr D$, where each $D\in \mathscr D$ must be thought of as the probability that an observation comes from $p^{\star}$ (the higher $D(x)$, the higher the probability that $x$ is drawn from $p^{\star}$). At some point, but not always, we will assume that $\mathscr D$ is in fact a parametric class, of the form $\{D_{\alpha}\}_{\alpha \in \Lambda}$, $\Lambda \subset \mathds R^q$, as is certainly always the case in practice. In GANs algorithms, both parametric models $\{G_{\theta}\}_{\theta \in \Theta}$ and $\{D_{\alpha}\}_{\alpha \in \Lambda}$ take the form of neural networks, but this does not play a fundamental role in this paper. We  will simply remember that the dimensions $p$ and $q$ are potentially very large, which takes us away from a classical parametric setting. We also insist on the fact that it is \underline{not} assumed that $p^{\star}$ belongs to $\mathscr P$.

Let $Z_1, \hdots, Z_n$ be an i.i.d.~sample of random variables, all distributed as the noise $Z$. The objective is to solve in $\theta$ the problem
$$\inf_{\theta \in \Theta}\sup_{D \in \mathscr D} \Big[\prod_{i=1}^n D(X_i)\times \prod_{i=1}^n (1-D\circ G_{\theta}(Z_i))\Big],$$
or, equivalently, to find $\hat \theta\in \Theta$ such that
\begin{equation}
\sup_{D \in \mathscr D}\hat L(\hat \theta, D) \leq \sup_{D \in \mathscr D}\hat L(\theta,D), \quad \forall \theta \in \Theta,
\label{GAN-likelihood}
\end{equation}
where
$$\hat L(\theta,D)\stackrel{\mbox{\tiny{def}}}{=}\sum_{i=1}^n \ln D(X_i)+ \sum_{i=1}^n\ln(1-D\circ G_{\theta}(Z_i))$$
($\ln$ is the natural logarithm). In this problem, $D(x)$ represents the probability that an observation $x$ comes from $p^{\star}$ rather than from $p_{\theta}$. Therefore, for each $\theta$, the discriminators (the police team) try to distinguish the original sample $X_1, \hdots, X_n$ from the fake one $G_{\theta}(Z_1), \hdots, G_{\theta}(Z_n)$ produced by the generators (the counterfeiters' team), by maximizing $D$ on the $X_i$ and minimizing it on the $G_{\theta}(Z_i)$. Of course, the generators have an exact opposite objective, and adapt the fake data in such a way as to mislead the discriminators' likelihood. All in all, we see that the criterion seeks to find the right balance between the conflicting interests of the generators and the discriminators. The hope is that the $\hat \theta$ achieving equilibrium will make it possible to generate observations $G_{\hat \theta}(Z_1), \hdots, G_{\hat \theta}(Z_n)$ indistinguishable from reality, i.e., observations with a law close to the unknown $p^{\star}$.
%
%

The criterion $\hat L(\theta,D)$ involved in (\ref{GAN-likelihood}) is the criterion originally proposed in the adversarial framework of \citet{GoPoMiXuWaOzCoBe14}. Since then, the success of GANs in applications has led to a large volume of literature on variants, which all have many desirable properties but are based on  different optimization criteria---examples are MMD-GANs \citep{DzRoGh15}, f-GANs \citep{NoCsTo16}, Wasserstein-GANs \citep{ArChBo17}, and an approach based on scattering transforms \citep{AnMa18}. All these variations and their innumerable algorithmic versions constitute the galaxy of GANs. That being said, despite increasingly spectacular applications, little is known about the mathematical and statistical forces behind these algorithms \citep[e.g.,][]{ArBo17,LiBoCh07,ZhLiZhXuHe18}, and, in fact, nearly nothing about the primary adversarial problem (\ref{GAN-likelihood}). As acknowledged by \citet{LiBoCh07}, basic questions on how well GANs can approximate the target distribution $p^{\star}$ remain largely unanswered. In particular, the role and impact of the discriminators on the quality of the approximation are still a mystery, and simple but fundamental questions regarding statistical consistency and rates of convergence remain open.

In the present article, we propose to take a small step towards a better theoretical understanding of GANs by analyzing some of the mathematical and statistical properties of the original adversarial problem (\ref{GAN-likelihood}). In Section \ref{OP}, we study the deep connection between the population version of (\ref{GAN-likelihood}) and the Jensen-Shannon divergence, together with some optimality characteristics of the problem, often referred to in the literature but in fact poorly understood. Section \ref{AP} is devoted to a better comprehension of the role of the discriminator family via approximation arguments. Finally, taking a statistical point of view, we study in Section \ref{SA} the large sample properties of the distribution $p_{\hat \theta}$ and $\hat \theta$, and prove in particular a central limit theorem for this parameter. Some of our results are illustrated with simulated examples. For clarity, most technical proofs are gathered in Section \ref{STL}.
\section{Optimality properties}
\label{OP}
We start by studying some important properties of the adversarial principle, emphasizing the role played by the Jensen-Shannon divergence. We recall that if $P$ and $Q$ are probability measures on $E$, and $P$ is absolutely continuous with respect to $Q$, then the Kullback-Leibler divergence from $Q$ to $P$ is defined as
$$\KL{P}{Q}=\int \ln \frac{{\rm d}P}{{\rm d}Q}{\rm d}P,$$
where $\frac{{\rm d}P}{{\rm d}Q}$ is the Radon-Nikodym derivative of $P$ with respect to $Q$.  The Kullback-Leibler divergence is always nonnegative, with $\KL{P}{Q}$ zero if and only if $P=Q$. If $p=\frac{{\rm d}P}{{\rm d}\mu}$ and $q=\frac{{\rm d}Q}{{\rm d}\mu}$ exist (meaning that $P$ and $Q$ are absolutely continuous with respect to $\mu$, with densities $p$ and $q$), then the Kullback-Leibler divergence is given as
$$\KL{P}{Q}=\int p\ln \frac{p}{q}{\rm d}\mu,$$
and alternatively denoted by $\KL{p}{q}$. We also recall that the Jensen-Shannon divergence is a symmetrized version of the Kullback-Leibler divergence. It is defined for any probability measures $P$ and $Q$ on $E$ by
$$\JS{P}{Q}=\frac{1}{2}\KL[\Big]{P}{\frac{P+Q}{2}}+\frac{1}{2}\KL[\Big]{Q}{\frac{P+Q}{2}},$$
and satisfies $0\leq \JS{P}{Q} \leq \ln 2$. The square root of the Jensen-Shannon divergence is a metric often referred to as Jensen-Shannon distance \citep{EnSc03}. When $P$ and $Q$ have densities $p$ and $q$ with respect to $\mu$, we use the notation $\JS{p}{q}$ in place of $\JS{P}{Q}$.

For a generator $G_{\theta}$ and an arbitrary discriminator $D\in \mathscr D$, the criterion $\hat L(\theta,D)$ to be optimized in (\ref{GAN-likelihood}) is but the empirical version of the probabilistic criterion
$$
L(\theta,D)\stackrel{\mbox{\tiny{def}}}{=}\int \ln(D) p^{\star}{\rm d}\mu+\int \ln (1-D) p_{\theta}{\rm d}\mu.
$$
We assume for the moment that the discriminator class $\mathscr D$ is not restricted and equals $\mathscr D_{\infty}$, the set of all Borel functions from $E$ to $[0,1]$. We note however that, for all $\theta \in \Theta$,
$$0\geq \sup_{D \in \mathscr D_{\infty}}L(\theta,D) \geq -\ln 2 \Big(\int p^{\star}{\rm d}\mu+\int  p_{\theta}{\rm d}\mu\Big)=-\ln 4,$$
so that $\inf_{\theta \in \Theta}\sup_{D \in \mathscr D_{\infty}} L(\theta,D) \in [-\ln 4,0]$. Thus,
$$\inf_{\theta \in \Theta}\sup_{D \in \mathscr D_{\infty}}L(\theta,D)=\inf_{\theta \in \Theta}\sup_{D \in \mathscr D_{\infty}: L(\theta,D)> -\infty}L(\theta,D).$$
This identity points out the importance of discriminators such that $L(\theta,D)> -\infty$, which we call $\theta$-admissible. In the sequel, in order to avoid unnecessary problems of integrability, we only consider such discriminators, keeping in mind that the others have no interest.

Of course, working with $\mathscr D_{\infty}$ is somehow an idealized vision, since in practice the discriminators are always parameterized by some parameter $\alpha \in \Lambda$, $\Lambda \subset \mathds R^q$. Nevertheless, this point of view is informative and, in fact, is at the core of the connection between our generative problem and the Jensen-Shannon divergence. Indeed, taking the supremum of $L(\theta,D)$ over $\mathscr D_{\infty}$, we have
\begin{align*}
\sup_{D \in \mathscr D_{\infty}}L(\theta,D)&=\sup_{D\in \mathscr D_{\infty}}\int \big[ \ln(D) p^{\star}+ \ln (1-D) p_{\theta}\big]{\rm d}\mu\\
& \leq \int \sup_{D\in \mathscr D_{\infty}}\big[ \ln(D) p^{\star}+ \ln (1-D) p_{\theta}\big]{\rm d}\mu\\
&=L(\theta,D_{\theta}^{\star}),
\end{align*}
where
\begin{align}
\label{optimD}
D_{\theta}^{\star}\stackrel{\mbox{\tiny{def}}}{=}\frac{p^{\star}}{p^{\star}+p_{\theta}}.
\end{align}
By observing that $L(\theta,D_{\theta}^{\star})=2\JS{p^{\star}}{p_{\theta}}-\ln 4$, we conclude that, for all $\theta \in \Theta$,
$$\sup_{D \in \mathscr D_{\infty}}L(\theta,D)=L(\theta,D_{\theta}^{\star})=2\JS{p^{\star}}{p_{\theta}}-\ln 4.$$
In particular, $D_{\theta}^{\star}$ is $\theta$-admissible. The fact that $D_{\theta}^{\star}$ realizes the supremum of $L(\theta,D)$ over $\mathscr D_{\infty}$ and that this supremum is connected to the Jensen-Shannon divergence between $p^{\star}$ and $p_{\theta}$ appears in the original article by \citet{GoPoMiXuWaOzCoBe14}. This remark has given rise to many developments that interpret the adversarial problem (\ref{GAN-likelihood}) as the empirical version of the minimization problem $\inf_{\theta}\JS{p^{\star}}{p_{\theta}}$ over $\Theta$. Accordingly, many GANs algorithms try to learn the optimal function $D_{\theta}^{\star}$, using for example stochastic gradient descent techniques and mini-batch approaches. However, it has not been known until now whether $D_{\theta}^{\star}$ is unique as a maximizer of $L(\theta,D)$ over all $D$. Our first result shows that this is indeed the case.
\begin{theo}
Let $\theta \in \Theta$ be such that $p_\theta >0$ $\mu$-almost everywhere. Then the function $D_{\theta}^{\star}$ is the unique discriminator that achieves the supremum of the functional $D\mapsto L(\theta,D)$ over $\mathscr D_{\infty}$, i.e.,
$$\{D_{\theta}^{\star}\}=\stackrel[D\in \mathscr D_{\infty}]{}{\arg\max} L(\theta,D).$$
\end{theo}
\begin{proof}
Let $D \in \mathscr D_{\infty}$ be a discriminator such that $L(\theta,D)=L(\theta,D_{\theta}^{\star})$. In particular, $L(\theta,D)>-\infty$ and $D$ is $\theta$-admissible. We have to show that $D=D_{\theta}^{\star}$. Notice that
\begin{equation}
\label{XXX}
\int \ln (D)p^{\star}{\rm d}\mu+\int \ln (1-D)p_{\theta}{\rm d}\mu = \int \ln (D_{\theta}^{\star})p^{\star}{\rm d}\mu+\int \ln (1-D_{\theta}^{\star})p_{\theta}{\rm d}\mu.
\end{equation}
Thus,
$$-\int \ln \Big(\frac{D_{\theta}^{\star}}{D}\Big)p^{\star}{\rm d}\mu=\int \ln \Big(\frac{1-D_{\theta}^{\star}}{1-D}\Big)p_{\theta}{\rm d}\mu,$$
i.e., by definition of $D_{\theta}^{\star}$,
\begin{equation}
\label{birs}
-\int \ln \Big(\frac{p^{\star}}{D(p^{\star}+p_{\theta})}\Big)p^{\star}{\rm d}\mu=\int \ln \Big(\frac{p_{\theta}}{(1-D)(p^{\star}+p_{\theta})}\Big)p_{\theta}{\rm d}\mu.
\end{equation}
Let ${\rm d}P^{\star}=p^{\star}{\rm d}\mu$, ${\rm d}P_{\theta}=p_{\theta}{\rm d}\mu$,
$${\rm d \kappa}=\frac{D(p^{\star}+p_{\theta})}{\int D(p^{\star}+p_{\theta}){\rm d}\mu}{\rm d}\mu,\quad \mbox{and}\quad {\rm d \kappa'}=\frac{(1-D)(p^{\star}+p_{\theta})}{\int (1-D)(p^{\star}+p_{\theta}){\rm d}\mu}{\rm d}\mu.$$
With this notation, identity (\ref{birs}) becomes
$$-\KL{P^{\star}}{\kappa}+\ln \Big[\int D(p^{\star}+p_{\theta}){\rm d}\mu \Big]=\KL{P_{\theta}}{\kappa'}-\ln \Big[\int (1-D)(p^{\star}+p_{\theta}){\rm d}\mu \Big].$$
Upon noting that
$$\int (1-D)(p^{\star}+p_{\theta}){\rm d}\mu=2-\int D(p^{\star}+p_{\theta}){\rm d}\mu,$$
we obtain
$$\KL{P^{\star}}{\kappa}+\KL{P_{\theta}}{\kappa'}=\ln\Big[\int D(p^{\star}+p_{\theta}){\rm d}\mu\big(2-\int D(p^{\star}+p_{\theta}){\rm d}\mu\big)\Big].$$
Since $\int D(p^{\star}+p_{\theta}){\rm d}\mu \in [0,2]$, we find that $\KL{P^{\star}}{\kappa}+\KL{P_{\theta}}{\kappa'}\leq 0$, which implies
$$\KL{P^{\star}}{\kappa}=0 \quad \mbox{and}\quad \KL{P_{\theta}}{\kappa'}=0.$$
Consequently,
$$p^{\star}=\frac{D(p^{\star}+p_{\theta})}{\int D(p^{\star}+p_{\theta}){\rm d}\mu}\quad \mbox{and}\quad p_{\theta}=\frac{(1-D)(p^{\star}+p_{\theta})}{2-\int D(p^{\star}+p_{\theta}){\rm d}\mu},$$
that is,
$$\int D(p^{\star}+p_{\theta}){\rm d}\mu=\frac{D(p^{\star}+p_{\theta})}{p^{\star}}\quad \mbox{and}\quad 1-D=\frac{p_{\theta}}{p^{\star}+p_{\theta}}\Big(2-\int D(p^{\star}+p_{\theta}){\rm d}\mu\Big).$$
We conclude that
$$1-D=\frac{p_{\theta}}{p^{\star}+p_{\theta}}\Big(2-\frac{D(p^{\star}+p_{\theta})}{p^{\star}}\Big),$$
i.e., $D=\frac{p^{\star}}{p^{\star}+p_{\theta}}$ whenever $p^{\star}\neq p_\theta$.

To complete the proof, it remains to show that $D=1/2$ $\mu$-almost everywhere on the set
$A\stackrel{\mbox{\tiny{def}}}{=}\{p_{\theta}=p^{\star}\}$. Using the result above together with equality (\ref{XXX}), we see that
$$\int_A \ln (D)p^{\star} {\rm d}\mu+\int_A \ln (1-D) p_\theta {\rm
d}\mu=\int_A \ln(1/2) p^{\star} {\rm d}\mu+\int_A \ln (1/2) p_\theta {\rm
d}\mu,$$
that is,
$$\int_A \big[\ln(1/4)-\ln(D(1-D))\big]p_{\theta} {\rm d}\mu=0.$$
Observing that $D(1-D)\le 1/4$ since $D$ takes values in $[0,1]$, we deduce that
$[\ln(1/4)-\ln(D(1-D))]p_{\theta} \mathbf 1_A=0$ $\mu$-almost everywhere. Therefore,
$D=1/2$ on the set $\{p_{\theta}=p^{\star}\}$, since $p_\theta >0$
$\mu$-almost everywhere by assumption.
\end{proof}
By definition of the optimal discriminator $D_{\theta}^{\star}$, we have
$$L(\theta,D_{\theta}^{\star})=\sup_{D \in \mathscr D_{\infty}}L(\theta,D)=2\JS{p^{\star}}{p_{\theta}}-\ln 4, \quad \forall \theta \in \Theta.$$
Therefore, it makes sense to let the parameter $\theta^{\star} \in \Theta$ be defined as
$$L(\theta^{\star},D^{\star}_{\theta^{\star}}) \leq L(\theta,D_{\theta}^{\star}), \quad \forall \theta \in \Theta,$$
or, equivalently,
\begin{equation}
\label{PA}
\JS{p^{\star}}{p_{\theta^{\star}}} \leq \JS{p^{\star}}{p_{\theta}}, \quad \forall \theta \in \Theta.
\end{equation}
The parameter $\theta^{\star}$ may be interpreted as the best parameter in $\Theta$ for approaching the unknown density $p^{\star}$ in terms of Jensen-Shannon divergence, in a context where all possible discriminators are available. In other words, the generator $G_{\theta^{\star}}$ is the ideal generator, and the density $p_{\theta^{\star}}$ is the one we would ideally like to use to generate fake samples. Of course, whenever $p^{\star} \in \mathscr P$ (i.e., the target density is in the model), then $p^{\star}=p_{\theta^{\star}}$, $\JS{p^{\star}}{p_{\theta^{\star}}}=0$, and $D_{\theta^{\star}}^{\star}=1/2$. This is, however, a very special case, which is of no interest, since in the applications covered by GANs, the data are usually so complex that the hypothesis $p^{\star} \in \mathscr P$ does not hold.

In the general case, our next theorem provides sufficient conditions for the existence and unicity of $\theta^{\star}$. For $P$ and $Q$ probability measures on $E$, we let $\delta(P,Q)=\sqrt{\JS{P}{Q}}$, and recall that $\delta$ is a distance on the set of probability measures on $E$ \citep{EnSc03}. We let ${\rm d}P^{\star}=p^{\star}{\rm d}\mu$ and, for all $\theta \in \Theta$, ${\rm d}P_{\theta}=p_{\theta}{\rm d}\mu$.
\begin{theo}
\label{E+U}
Assume that the model $\{P_{\theta}\}_{\theta \in \Theta}$ is identifiable, convex, and compact for the metric $\delta$. Assume, in addition, that there exist $0<m\leq M$ such that $m\leq p^{\star}\leq M$ and, for all $\theta \in \Theta$, $p_{\theta}\leq M$. Then there exists a unique $\theta^{\star} \in \Theta$ such that
$$\{\theta^{\star}\}=\stackrel[\theta \in \Theta]{}{\arg\min}L(\theta,D^{\star}_{\theta}),$$
or, equivalently,
$$\{\theta^{\star}\}=\stackrel[\theta \in \Theta]{}{\arg\min} \JS{p^{\star}}{p_{\theta}}.$$
\end{theo}
\begin{proof}
Observe that $\mathscr P=\{p_{\theta}\}_{\theta\in \Theta} \subset L^1(\mu)\cap L^2(\mu)$ since $0\leq p_{\theta}\leq M$ and $\int p_{\theta}{\rm d}\mu=1$. Recall that $L(\theta,D_{\theta}^{\star})=\sup_{D \in \mathscr D_{\infty}}L(\theta,D)=2\JS{p^{\star}}{p_{\theta}}-\ln 4$. By identifiability of $\{P_{\theta}\}_{\theta \in \Theta}$, it is enough to prove that there exists a unique density $p_{\theta^{\star}}$ of $\mathscr P$ such that
$$\{p_{\theta^{\star}}\}=\stackrel[p \in \mathscr P]{}{\arg \min}\JS{p^{\star}}{p}.$$
{\bf Existence}. Since $\{P_{\theta}\}_{\theta \in \Theta}$ is compact for $\delta$, it is enough to show that the function
\begin{equation*}
\begin{array}{lll}
\{P_{\theta}\}_{\theta \in \Theta} &\to & \mathds R_+\\
P &\mapsto &\JS{P^{\star}}{P}
\end{array}
\end{equation*}
is continuous. But this is clear since, for all $P_1,P_2 \in \{P_{\theta}\}_{\theta \in \Theta}$, $|\delta(P^{\star},P_1)-\delta (P^{\star},P_2)| \leq \delta(P_1,P_2)$ by the triangle inequality. Therefore, ${\arg \min}_{p \in \mathscr P}\JS{p^{\star}}{p} \neq \emptyset$.

{\bf Unicity}. For $a \in [m,M]$, we consider the function $F_a$ defined by
$$F_a(x)=a\ln\Big(\frac{2a}{a+x}\Big)+x\ln\Big(\frac{2x}{a+x}\Big), \quad x\in [0,M],$$
with the convention $0\ln0=0$. Clearly, $F''_a(x)=\frac{a}{x(a+x)}\geq \frac{m}{2M^2}$, which shows that $F_a$ is $\beta$-strongly convex, with $\beta>0$ independent of $a$. Thus, for all $\lambda \in [0,1]$, all $x_1,x_2 \in [0,M]$, and $a\in [m,M]$,
$$F_a(\lambda x_1+(1-\lambda)x_2)\leq \lambda F_a(x_1)+(1-\lambda)F_a(x_2)-\frac{\beta}{2}\lambda(1-\lambda)(x_1-x_2)^2.$$
Thus, for all $p_1,p_2 \in \mathscr P$ with $p_1\neq p_2$, and for all $\lambda \in (0,1)$,
\begin{align*}
&\JS{p^{\star}}{\lambda p_1+(1-\lambda)p_2} \\
&\quad =\int F_{p^{\star}}(\lambda p_1+(1-\lambda)p_2){\rm d}\mu\\
& \quad \leq \lambda \JS{p^{\star}}{p_1}+(1-\lambda)\JS{p^{\star}}{p_2}-\frac{\beta}{2} \lambda(1-\lambda)\int(p_1-p_2)^2{\rm d}\mu\\
& \quad < \lambda \JS{p^{\star}}{p_1}+(1-\lambda)\JS{p^{\star}}{p_2}.
\end{align*}
In the last inequality, we used the fact that $\frac{\beta}{2} \lambda(1-\lambda)\int(p_1-p_2)^2{\rm d}\mu$ is positive and finite since $p_{\theta}\in L^2(\mu)$ for all $\theta$. We conclude that the function $L^1({\mu}) \supset  \mathscr P\ni p \mapsto \JS{p^{\star}}{p}$ is strictly convex. Therefore, its $\arg\min$ is either the empty set or a singleton.
\end{proof}
\begin{rem}
There are simple conditions for the model $\{P_{\theta}\}_{\theta \in \Theta}$ to be compact for the metric $\delta$. It is for example enough to suppose that $\Theta$ is compact, $\{P_{\theta}\}_{\theta \in \Theta}$ is convex, and
\begin{enumerate}[$(i)$]
\item For all $x\in E$, the function $\theta\mapsto p_{\theta}(x)$ is continuous on $\Theta$;
\item One has $\sup_{(\theta,\theta') \in \Theta^2}|p_{\theta}\ln p_{\theta'}| \in L^1(\mu)$.
\end{enumerate}
Let us quickly check that under these conditions, $\{P_{\theta}\}_{\theta \in \Theta}$ is compact for the metric $\delta$. Since $\Theta$ is compact, by the sequential characterization of compact sets, it is enough to prove that if $\Theta \supset (\theta_n)_n$ converges to $\theta \in \Theta$, then $\JS{p_{\theta}}{p_{\theta_n}}\to 0$. But,
$$\JS{p_{\theta}}{p_{\theta_n}}=\int \Big[p_{\theta}\ln \Big(\frac{2p_{\theta}}{p_{\theta}+p_{\theta_n}}\Big)+p_{\theta_n}\ln\Big(\frac{2p_{\theta_n}}{p_{\theta}+p_{\theta_n}}\Big)\Big]{\rm d}\mu.$$
By the convexity of $\{P_{\theta}\}_{\theta \in \Theta}$, using $(i)$ and $(ii)$, the Lebesgue dominated convergence theorem shows that $\JS{p_{\theta}}{p_{\theta_n}}\to 0$, whence the result.
\end{rem}
Interpreting the adversarial problem in connection with the optimization program $\inf_{\theta \in \Theta}\JS{p^{\star}}{p_{\theta}}$ is a bit misleading, because this is based on the assumption that all possible discriminators are available (and in particular the optimal discriminator $D_{\theta}^{\star}$). In the end this means assuming that we know the distribution $p^{\star}$, which is eventually not acceptable from a statistical perspective. In practice, the class of discriminators is always restricted to be a parametric family $\mathscr D=\{D_{\alpha}\}_{\alpha \in \Lambda}$, $\Lambda \subset \mathds R^q$, and it is with this class that we have to work. From our point of view, problem (\ref{GAN-likelihood}) is a likelihood-type problem involving two parametric families $\mathscr G$ and $\mathscr D$, which must be analyzed as such, just as we would do for a classical maximum likelihood approach. In fact, it takes no more than a moment's thought to realize that the key lies in the approximation capabilities of the discriminator class $\mathscr D$ with respect to the functions $D_{\theta}^{\star}$, $\theta \in \Theta$. This is the issue that we discuss in the next section.
\section{Approximation properties}
\label{AP}
In the remainder of the article, we assume that $\theta^{\star}$ exists, keeping in mind that Theorem \ref{E+U} provides us with precise conditions guaranteeing its existence and its unicity. As pointed out earlier, in practice only a parametric class $\mathscr D=\{D_{\alpha}\}_{\alpha \in \Lambda}$, $\Lambda \subset \mathds R^q$, is available, and it is therefore logical to consider the parameter $\bar \theta \in \Theta$ defined by
$$\sup_{D \in \mathscr D}L(\bar \theta,D) \leq \sup_{D \in \mathscr D}L(\theta,D), \quad \forall \theta \in \Theta.$$
(We assume for now that $\bar \theta$ exists---sufficient conditions for this existence, relating to compactness of $\Theta$ and regularity of the model $\mathscr P$, will be given in the next section.) The density $p_{\bar \theta}$ is thus the best candidate to imitate $p_{\theta^{\star}}$, given the parametric families of generators $\mathscr G$ and discriminators $\mathscr D$. The natural question is then: is it possible to quantify the proximity between $p_{\bar \theta}$ and the ideal $p_{\theta^{\star}}$ via the approximation properties of the class $\mathscr D$? In other words, if $\mathscr D$ is growing, is it true that $p_{\bar \theta}$ approaches $p_{\theta^{\star}}$, and in the affirmative, in which sense and at which speed? Theorem \ref{theorem-encadrement} below provides a first answer to this important question, in terms of the difference $\JS{p^{\star}}{p_{\bar \theta}}-\JS{p^{\star}}{p_{\theta^{\star}}}$. To state the result, we will need some assumptions.

{\bf Assumption} $(H_0)$ There exists a positive constant $\underline t \in (0,1/2]$ such that
$$\min (D_{\theta}^{\star} ,1-D_{\theta}^{\star})  \geq \underline t, \quad \forall \theta \in \Theta.$$

We note that this assumption implies that, for all $\theta \in \Theta$,
$$\frac{\underline t}{1-\underline t}p^{\star} \leq p_{\theta} \leq \frac{1-\underline t}{\underline t}p^{\star}.$$
It is a mild requirement, which implies in particular that for any $\theta$, $p_{\theta}$ and $p^{\star}$ have the same support, independent of $\theta$.

Let $\|\cdot\|_{\infty}$ be the supremum norm of functions on $E$. Our next condition guarantees that the parametric class $\mathscr D$ is rich enough to approach the discriminator $D^{\star}_{\bar \theta}$.

{\bf Assumption} $(H_{\varepsilon})$ There exists $\varepsilon \in (0,\underline t)$ and $D \in \mathscr D$, a $\bar \theta$-admissible discriminator, such that $\|D-D_{\bar \theta}^{\star}\|_{\infty}\leq \varepsilon$.

We are now equipped to state our approximation theorem. For ease of reading, its proof is postponed to Section \ref{STL}.
\begin{theo}
\label{theorem-encadrement}
Under Assumptions $(H_0)$ and $(H_{\varepsilon})$, there exists a positive constant $c$  (depending only upon $\underline t$) such that
\begin{equation}
\label{encadrement}
0 \leq \JS{p^{\star}}{p_{\bar \theta}}-\JS{p^{\star}}{p_{\theta^{\star}}}\leq c \varepsilon^2.
\end{equation}
\end{theo}
This theorem points out that if the class $\mathscr D$ is rich enough to approximate the discriminator $D_{\bar \theta}^{\star}$ in such a way that $\|D-D_{\bar \theta}^{\star}\|_{\infty}\leq \varepsilon$ for some small $\varepsilon$, then replacing $\JS{p^{\star}}{p_{\theta^{\star}}}$ by $\JS{p^{\star}}{p_{\bar \theta}}$ has an impact which is not larger than a ${\rm O}(\varepsilon^2)$ factor. It shows in particular that the Jensen-Shannon divergence is a suitable criterion for the problem we are examining.
\section{Statistical analysis}
\label{SA}
The data-dependent parameter $\hat \theta$ achieves the infimum of the adversarial problem (\ref{GAN-likelihood}). Practically speaking, it is this parameter that will be used in the end for producing fake data, via the associated generator $G_{\hat \theta}$. We first study in Subsection \ref{AP1} the large sample properties of the distribution $p_{\hat \theta}$ via the criterion $\JS{p^{\star}}{p_{\hat \theta}}$, and then state in Subsection \ref{AP2} the almost sure convergence and asymptotic normality of the parameter $\hat \theta$ as the sample size $n$ tends to infinity. Throughout, the parameter sets $\Theta$ and $\Lambda$ are assumed to be compact subsets of $\mathds R^p$ and $\mathds R^q$, respectively. To simplify the analysis, we also assume that $\mu(E)<\infty$.
\subsection{Asymptotic properties of $\JS{p^{\star}}{p_{\hat \theta}}$}
\label{AP1}
As for now, we assume that we have at hand a parametric family of generators $\mathscr G=\{G_{\theta}\}_{\theta \in \Theta}$, $\Theta \subset \mathds R^p$, and a parametric family of discriminators $\mathscr D=\{D_{\alpha}\}_{\alpha \in \Lambda}$, $\Lambda \subset \mathds R^q$. We recall that the collection of probability densities associated with $\mathscr G$ is $\mathscr P=\{p_{\theta}\}_{\theta \in \Theta}$, where $G_{\theta}(Z)\stackrel{\mathscr L}{=}p_{\theta}{\rm d}\mu$ and $Z$ is some low-dimensional noise random variable. In order to avoid any confusion, for a given discriminator $D=D_{\alpha}$ we use the notation $\hat L(\theta,\alpha)$ (respectively, $L(\theta,\alpha)$)
instead of $\hat L(\theta,D)$ (respectively, $L(\theta,D)$) when useful. So,
$$
\hat L(\theta,\alpha)=\sum_{i=1}^n \ln D_{\alpha}(X_i)+ \sum_{i=1}^n\ln(1-D_{\alpha}\circ G_{\theta}(Z_i)),
$$
and
$$
L(\theta,\alpha)=\int \ln(D_{\alpha}) p^{\star}{\rm d}\mu+\int \ln (1-D_{\alpha}) p_{\theta}{\rm d}\mu.
$$
We will need the following regularity assumptions:

{\bf Assumptions} $(H_{\text{reg}})$
\begin{enumerate}
\item[$(H_D)$] There exists $\kappa \in (0,1/2)$ such that, for all $\alpha \in \Lambda$, $\kappa \leq D_{\alpha}\leq 1-\kappa$. In addition, the function $(x,\alpha) \mapsto D_{\alpha}(x)$ is of class $C^1$, with a uniformly bounded differential.
\item[$(H_G)$] For all $z \in \mathds R^{d'}$, the function $\theta \mapsto G_{\theta}(z)$ is of class $C^1$, uniformly bounded, with a uniformly bounded differential.
\item[$(H_p)$] For all $x \in E$, the function $\theta \mapsto p_{\theta}(x)$ is of class $C^1$, uniformly bounded, with a uniformly bounded differential.
\end{enumerate}
Note that under $(H_D)$, all discriminators in $\{D_{\alpha}\}_{\alpha \in \Lambda}$ are $\theta$-admissible, whatever $\theta$. All of these requirements are classic regularity conditions for statistical models, which imply in particular that the functions $\hat L(\theta,\alpha)$ and $L(\theta,\alpha)$ are continuous. Therefore, the compactness of $\Theta$ guarantees that $\hat \theta$ and $\bar \theta$ exists. Conditions for the existence of $\theta^{\star}$ are given in Theorem \ref{E+U}.

We have known since Theorem \ref{theorem-encadrement} that if the available class of discriminators $\mathscr D$ approaches the optimal discriminator $D_{\bar \theta}^{\star}$ by a distance not more than $\varepsilon$, then $\JS{p^{\star}}{p_{\bar \theta}}-\JS{p^{\star}}{p_{\theta^{\star}}}={\rm O}(\varepsilon^2)$. It is therefore reasonable to expect that, asymptotically, the difference $\JS{p^{\star}}{p_{\hat \theta}}-\JS{p^{\star}}{p_{\theta^{\star}}}$ will not be larger than a term proportional to $\varepsilon^2$, in some probabilistic sense. This is precisely the result of Theorem \ref{rate} below. In fact, most articles to date have focused on the development and analysis of optimization procedures (typically, stochastic-gradient-type algorithms) to compute $\hat \theta$, without really questioning its convergence properties as the data set grows. Although our statistical results are theoretical in nature, we believe that they are complementary to the optimization literature, insofar as they offer guarantees on the validity of the algorithms.

In addition to the regularity hypotheses and Assumption $(H_0)$, we will need the following requirement, which is a stronger version of $(H_{\varepsilon})$:

{\bf Assumption} $(H_{\varepsilon}')$  There exists $\varepsilon \in (0,\underline t)$ such that: for all $\theta \in \Theta$, there exists $D \in \mathscr D$, a $\theta$-admissible discriminator,  such that $\|D-D_{\theta}^{\star}\|_{\infty}\leq \varepsilon$.

We are ready to state our first statistical theorem.
\begin{theo}
\label{rate}
Under Assumptions $(H_0)$, $(H_{\emph{reg}})$, and $(H_{\varepsilon}')$, one has
$$\mathds E \JS{p^{\star}}{p_{\hat \theta}}-\JS{p^{\star}}{p_{\theta^{\star}}}={\rm O}\Big(\varepsilon^2 +\frac{1}{\sqrt n}\Big).$$
\end{theo}
\begin{proof}
Fix $\varepsilon  \in (0,\underline t)$ as in Assumption $(H'_{\varepsilon})$, and choose $\hat D \in \mathscr D$, a $\hat \theta$-admissible discriminator, such that $\|\hat D-D_{\hat \theta}^{\star}\|_{\infty}\leq \varepsilon$. By repeating the arguments of the proof of Theorem \ref{theorem-encadrement} (with $\hat \theta$ instead of $\bar \theta$), we conclude that there exists a constant $c_1>0$ such that
$$
2\JS{p^{\star}}{p_{\hat \theta}} \leq c_1 \varepsilon^2+L(\hat \theta,\hat D)+\ln 4 \leq c_1 \varepsilon^2+\sup_{\alpha \in \Lambda} L(\hat \theta,\alpha)+\ln 4.
$$
Therefore,
\begin{align*}
2 \JS{p^{\star}}{p_{\hat \theta}} & \leq c_1\varepsilon^2+\sup_{\theta \in \Theta,\alpha \in \Lambda}|\hat L(\theta,\alpha)-L(\theta,\alpha)|+\sup_{\alpha \in \Lambda}\hat L(\hat \theta,\alpha)+\ln 4\\
&=c_1\varepsilon^2+\sup_{\theta \in \Theta,\alpha \in \Lambda}|\hat L(\theta,\alpha)-L(\theta,\alpha)|+\inf_{\theta \in \Theta}\sup_{\alpha \in \Lambda}\hat L(\theta,\alpha)+\ln 4\\
& \quad \mbox{(by definition of $\hat \theta$)}\\
& \leq c_1\varepsilon^2+2\sup_{\theta \in \Theta,\alpha \in \Lambda}|\hat L(\theta,\alpha)-L(\theta,\alpha)|+\inf_{\theta \in \Theta}\sup_{\alpha \in \Lambda}L(\theta,\alpha)+\ln 4.
\end{align*}
So,
\begin{align*}
2 \JS{p^{\star}}{p_{\hat \theta}} &\leq c_1\varepsilon^2+2\sup_{\theta \in \Theta,\alpha \in \Lambda}|\hat L(\theta,\alpha)-L(\theta,\alpha)|+\inf_{\theta \in \Theta}\sup_{D \in \mathscr D_{\infty}}L(\theta,D)+\ln 4\\
&=c_1\varepsilon^2+2\sup_{\theta \in \Theta,\alpha \in \Lambda}|\hat L(\theta,\alpha)-L(\theta,\alpha)|+L(\theta^{\star},D^{\star}_{\theta^{\star}})+\ln 4\\
& \quad \mbox{(by definition of $\theta^{\star}$)}\\
&=c_1\varepsilon^2+2 \JS{p^{\star}}{p_{\theta^{\star}}}+2\sup_{\theta \in \Theta,\alpha \in \Lambda}|\hat L(\theta,\alpha)-L(\theta,\alpha)|.
\end{align*}
Thus, letting $c_2=c_1/2$, we have
\begin{equation}
\label{inter1}
\JS{p^{\star}}{p_{\hat \theta}}-\JS{p^{\star}}{p_{\theta^{\star}}}\leq c_2\varepsilon^2+\sup_{\theta \in \Theta,\alpha \in \Lambda}|\hat L(\theta,\alpha)-L(\theta,\alpha)|.
\end{equation}
Clearly, under Assumptions $(H_D)$, $(H_G)$, and $(H_p)$, the process $(\hat L(\theta,\alpha)-L(\theta,\alpha))_{\theta \in \Theta,\alpha \in \Lambda}$ is subgaussian \citep[e.g.,][Chapter 5]{Ha16} for the distance $d=\|\cdot\|/\sqrt{n}$, where $\|\cdot\|$ is the standard Euclidean norm on $\mathds R^p\times \mathds R^q$. Let $N(\Theta \times \Lambda,\|\cdot\|,u)$ denote the $u$-covering number of $\Theta \times \Lambda$ for the distance $\|\cdot\|$. Then, by Dudley's inequality \citep[][Corollary 5.25]{Ha16},
\begin{equation}
\label{Dudley}
\mathds E \sup_{\theta \in \Theta,\alpha \in \Lambda}|\hat L(\theta,\alpha)-L(\theta,\alpha)| \leq \frac{12}{\sqrt n}\int_0^{\infty}\sqrt{\ln (N(\Theta \times \Lambda,\|\cdot\|,u))}{\rm d}u.
\end{equation}
Since $\Theta$ and $\Lambda$ are bounded, there exists $r>0$ such that $N(\Theta \times \Lambda,\|\cdot\|,u)=1$ for $u \geq r$ and
$$N(\Theta \times \Lambda,\|\cdot\|,u)={\rm O}\Big(\Big(\frac{1}{u}\Big)^{p+q}\Big) \quad \mbox{for }u< r.$$
Combining this inequality with (\ref{inter1}) and (\ref{Dudley}), we obtain
$$\mathds E \JS{p^{\star}}{p_{\hat \theta}}-\JS{p^{\star}}{p_{\theta^{\star}}}\leq c_3\Big(\varepsilon^2+\frac{1}{\sqrt n}\Big),$$
for some positive constant $c_3$. The conclusion follows by observing that, by (\ref{PA}),
$$\JS{p^{\star}}{p_{\theta^{\star}}} \leq \JS{p^{\star}}{p_{\hat \theta}}.$$
\end{proof}
Theorem \ref{rate} is illustrated in Figure \ref{img:comparison_of_dvgs_wrt_number_of_layers}, which shows the approximate values of $\mathds E \JS{p^{\star}}{p_{\hat \theta}}$. We took $p^{\star}(x)= \frac{e^{-x/s}}{s(1+e^{-x/s})^2}$ (centered logistic density with scale parameter $s=0.33$), and let $\mathscr G$ and $\mathscr D$ be two fully connected neural networks parameterized by weights and offsets. The noise random variable $Z$ follows a uniform distribution on $[0, 1]$, and the parameters of $\mathscr G$ and $\mathscr D$ are chosen in a sufficiently large compact set. In order to illustrate the impact of $\varepsilon$ in Theorem \ref{rate}, we fixed the sample size to a large $n=100\,000$ and varied the number of layers of the discriminators from 2 to 5, keeping in mind that a larger number of layers results in a smaller $\varepsilon$. To diversify the setting, we also varied the number of layers of the generators from 2 to 3.
The expectation $\mathds E \JS{p^{\star}}{p_{\hat \theta}}$ was estimated by averaging over 30 repetitions (the number of runs has been reduced for time complexity limitations). Note that we do not pay attention to the exact value of the constant term $\JS{p^{\star}}{p_{\theta^{\star}}}$, which is intractable in our setting.
\begin{figure}[!!h]
  \center
  \includegraphics[width=1\textwidth]{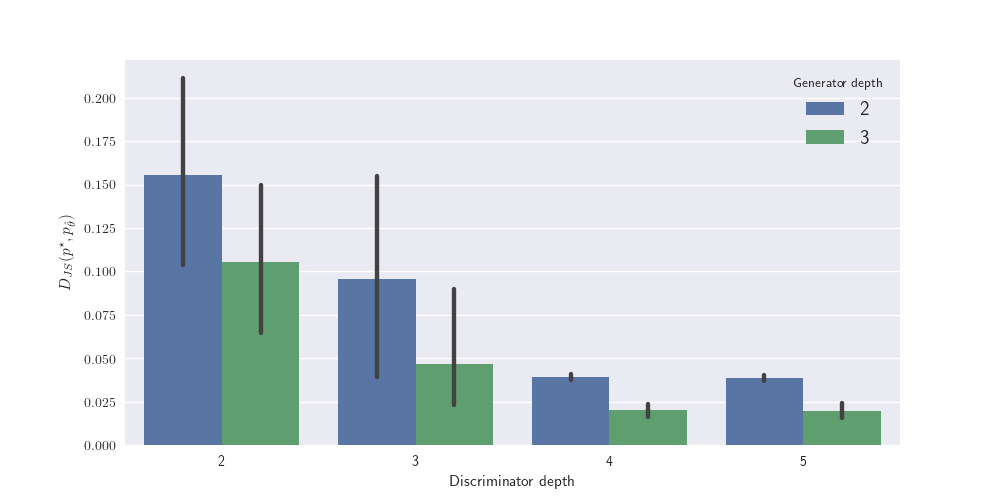}
  \caption{Bar plots of the Jensen-Shannon divergence $\JS{p^{\star}}{p_{\hat \theta}}$ with respect to the number of layers (depth) of both the discriminators and generators. The height of each rectangle estimates $\mathds E \JS{p^{\star}}{p_{\hat \theta}}.$
 } 
  \label{img:comparison_of_dvgs_wrt_number_of_layers}
\end{figure}

Figure \ref{img:comparison_of_dvgs_wrt_number_of_layers} highlights that $\mathds E \JS{p^{\star}}{p_{\hat \theta}}$ approaches the constant value $\JS{p^{\star}}{p_{\theta^{\star}}}$ as $\varepsilon \downarrow 0$, i.e., as the discriminator depth increases, given that the contribution of $1/\sqrt n$ is certainly negligible for $n=100\,000$. Figure \ref{img:examples_of_fitting} shows the target density $p^{\star}$ vs.~the histograms and kernel estimates of $100\,000$ data sampled from $G_{\hat \theta}(Z)$, in the two cases: (discriminator depth = 2, generator depth = 3) and (discriminator depth = 5, generator depth = 3). In accordance with the decrease of $\mathds E \JS{p^{\star}}{p_{\hat\theta}}$, the estimation of the true distribution $p^{\star}$ improves when $\varepsilon$ becomes small.
\begin{figure}[!!h]
    \center
    \subfloat[Discriminator depth = 2, generator depth = 3.]{{\includegraphics[width=8cm,height=4cm]{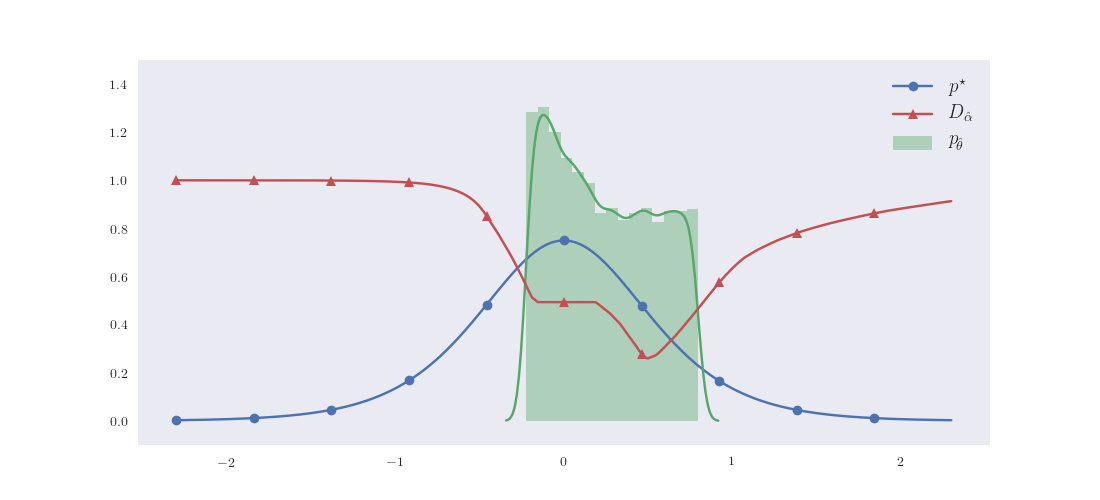}}}%
    \subfloat[Discriminator depth = 5, generator depth = 3.]{{\includegraphics[width=8cm,height=4cm]{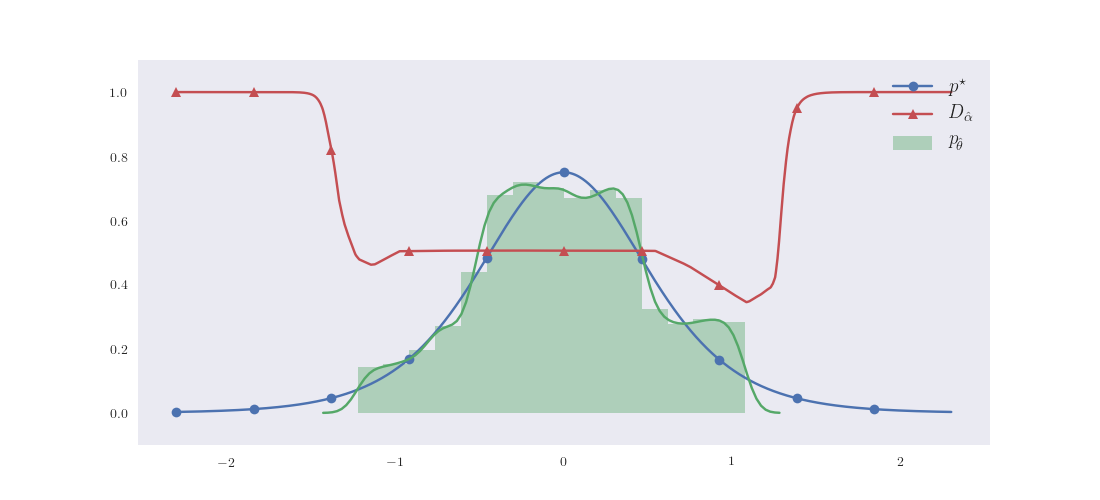}}}%
    \caption{True density $p^\star$, histograms, and kernel estimates (continuous line) of $100\,000$ data sampled from $G_{\hat \theta}(Z)$. Also shown is the final discriminator $D_{\hat \alpha}$.}
    \label{img:examples_of_fitting}%
\end{figure}

\paragraph{Some comments on the optimization scheme.} Numerical optimization is quite a tough point for GANs, partly due to nonconvex-concavity of the saddle point problem described in equation (\ref{GAN-likelihood}) and the nondifferentiability of the objective function. This motivates a very active line of research \citep[e.g.,][]{GoPoMiXuWaOzCoBe14,NoCsTo16,ArChBo17,ArBo17}, which aims at transforming the objective into a more convenient function and devising efficient algorithms. In the present paper, since we are interested in original GANs, the algorithmic approach described by \citet{GoPoMiXuWaOzCoBe14} is adopted, and numerical optimization is performed thanks to the machine learning framework \verb+TensorFlow+, working with gradient descent based on automatic differentiation. As proposed by \citet{GoPoMiXuWaOzCoBe14}, the objective function $\theta \mapsto \sup_{\alpha \in \Lambda} \hat L (\theta, \alpha)$ is not directly minimized. We used instead an alternated procedure, which consists in iterating (a few hundred times in our examples) the following two steps:
  \begin{enumerate}[$(i)$]
    \item For a fixed value of $\theta$ and from a given value of $\alpha$, perform 10 ascent steps on
    $\hat L (\theta, \cdot)$;
    \item For a fixed value of $\alpha$ and from a given value of $\theta$, perform 1 descent step on
     $\theta \mapsto -\sum_{i=1}^n \ln(D_\alpha \circ G_{\theta}(Z_i))$ (instead of $\theta \mapsto \sum_{i=1}^n \ln(1-D_\alpha \circ G_{\theta}(Z_i))$).
  \end{enumerate}
This alternated procedure is motivated by two reasons. First, for a given $\theta$, approximating $\sup_{\alpha \in \Lambda} \hat L (\theta, \alpha)$ is computationally prohibitive and may result in overfitting the finite training sample \citep{GoPoMiXuWaOzCoBe14}. This can be explained by the shape of the function $\theta\mapsto \sup_{\alpha \in \Lambda} \hat L (\theta, \alpha)$, which may be almost piecewise constant, resulting in a zero gradient almost everywhere (or at best very low; see \citealp{ArChBo17}). Next,
empirically, $-\ln(D_\alpha \circ G_{\theta}(Z_i))$ provides bigger gradients than $\ln(1-D_\alpha \circ G_{\theta}(Z_i))$, resulting in a more powerful algorithm than the original version, while leading to the same minimizers.

In all our experiments, the learning rates needed in gradient steps were fixed and tuned by hand, in order to prevent divergence. In addition, since our main objective is to focus on illustrating the statistical properties of GANs rather than delving into optimization issues, we decided to perform mini-batch gradient updates instead of stochastic ones (that is, new observations of $X$ and $Z$ are not sampled at each step of the procedure).  This is different of what is done in the original algorithm of \citet{GoPoMiXuWaOzCoBe14}. 

We realize that our numerical approach---although widely adopted by the machine learning community---may fail to locate the desired estimator $\hat \theta$ (i.e., the exact minimizer in $\theta$ of $\sup_{\alpha \in \Lambda} \hat L (\theta, \alpha)$) in more complex contexts than those presented in the present paper. It is nevertheless sufficient for our objective, which is limited to illustrating the theoretical results with a few simple examples.
\subsection{Asymptotic properties of $\hat \theta$}
\label{AP2}
Theorem \ref{rate} states a result relative to the criterion $\JS{p^{\star}}{p_{\hat \theta}}$. We now examine the convergence properties of the parameter $\hat \theta$ itself as the sample size $n$ grows. We would typically like to find reasonable conditions ensuring that $\hat \theta \to \bar \theta$ almost surely as $n\to \infty$. To reach this goal, we first need to strengthen a bit the Assumptions $(H_{\text{reg}})$, as follows:

{\bf Assumptions} $(H'_{\text{reg}})$
\begin{enumerate}
\item[$(H'_D)$] There exists $\kappa\in (0,1/2)$ such that, for all $\alpha \in \Lambda$, $\kappa \leq D_{\alpha}\leq 1-\kappa$. In addition, the function $(x,\alpha)\mapsto D_{\alpha}(x)$ is of class $C^2$, with differentials of order 1 and 2 uniformly bounded. 
\item[$(H'_G)$] For all $z \in \mathds R^{d'}$, the function $\theta \mapsto G_{\theta}(z)$ is of class $C^2$, uniformly bounded, with differentials of order 1 and 2 uniformly bounded.
\item[$(H'_p)$] For all $x \in E$, the function $\theta \mapsto p_{\theta}(x)$ is of class $C^2$, uniformly bounded, with differentials of order 1 and 2 uniformly bounded.
\end{enumerate}
It is easy to verify that under these assumptions the partial functions $\theta \mapsto \hat L(\theta,\alpha)$ (respectively, $\theta \mapsto L(\theta,\alpha)$) and $\alpha\mapsto \hat L(\theta,\alpha)$ (respectively, $\alpha \mapsto L(\theta,\alpha)$) are of class $C^2$. Throughout, we let $\theta=(\theta_1, \hdots, \theta_p)$, $\alpha=(\alpha_1, \hdots, \alpha_q)$, and denote by $\frac{\partial}{\partial \theta_i}$ and $\frac{\partial}{\partial\alpha_j}$ the partial derivative operations with respect to $\theta_i$ and $\alpha_j$. The next lemma will be of constant utility. In order not to burden the text, its proof is given in Section \ref{STL}.
\begin{lem}
\label{differentials}
Under Assumptions $(H'_{\emph{reg}})$, $\forall (a,b,c,d) \in \{0,1,2\}^4$ such that $a+b \leq 2$ and $c+d\leq 2$, one has
$$\sup_{\theta \in \Theta,\alpha \in \Lambda}\bigg| \frac{\partial ^{a+b+c+d}}{\partial \theta_i^a \partial \theta_j^b \partial \alpha_{\ell}^c \partial \alpha_m^d} \hat L(\theta,\alpha)-\frac{\partial ^{a+b+c+d}}{\partial \theta_i^a \partial \theta_j^b \partial \alpha_{\ell}^c \partial \alpha _m^d}L(\theta,\alpha)\bigg| \to 0 \quad \mbox{almost surely},$$
for all $(i,j) \in \{1, \hdots, p\}^2$ and $(\ell,m) \in \{1, \hdots, q\}^2$.
\end{lem}
We recall that $\bar \theta \in \Theta$ is such that
$$\sup_{\alpha \in \Lambda}L(\bar \theta,\alpha) \leq \sup_{\alpha \in \Lambda} L(\theta,\alpha), \quad \forall \theta \in \Theta,$$
and insist that $\bar \theta$ exists under $(H'_{\text{reg}})$ by continuity of the function $\theta \mapsto \sup_{\alpha \in \Lambda}L(\theta,\alpha)$. Similarly, there exists $\bar \alpha \in \Lambda$ such that
$$L(\bar \theta,\bar \alpha)\geq L(\bar \theta,\alpha), \quad \forall \alpha \in \Lambda.$$
The following assumption ensures that $\bar \theta$ and $\bar \alpha$ are uniquely defined, which is of course a key hypothesis for our estimation objective. Throughout, the notation $S^{\circ}$ (respectively, $\partial S$) stands for the interior (respectively, the boundary) of the set $S$.

{\bf Assumption $(H_1)$} The pair $(\bar \theta,\bar \alpha)$ is unique and belongs to $\Theta^{\circ} \times \Lambda^{\circ}$.

Finally, in addition to $\hat \theta$, we let $\hat \alpha \in \Lambda$ be such that
$$\hat L(\hat \theta,\hat \alpha) \geq \hat L(\hat \theta,\alpha), \quad \forall \alpha \in \Lambda.$$
\begin{theo}
\label{consistency}
Under Assumptions $(H'_{\emph{reg}})$ and $(H_1)$, one has
$$\hat \theta \to \bar \theta \quad \mbox{almost surely} \quad \mbox{and} \quad \hat \alpha \to \bar \alpha\quad \mbox{almost surely}.$$
\end{theo}
\begin{proof}
We write
\begin{align*}
&|\sup_{\alpha \in \Lambda} L(\hat \theta,\alpha)-\sup_{\alpha \in \Lambda} L(\bar \theta,\alpha)| \\
& \quad  \leq |\sup_{\alpha \in \Lambda} L(\hat \theta,\alpha)-\sup_{\alpha \in \Lambda} \hat L(\hat \theta,\alpha)|+ |\inf_{\theta \in \Theta}\sup_{\alpha \in \Lambda}\hat L(\theta,\alpha)-\inf_{\theta \in \Theta}\sup_{\alpha \in \Lambda} L(\theta,\alpha)|\\
& \quad \leq 2\sup_{\theta \in \Theta, \alpha \in \Lambda}|\hat L(\theta,\alpha)-L(\theta,\alpha)|.
\end{align*}
Thus, by Lemma \ref{differentials}, $\sup_{\alpha \in \Lambda} L(\hat \theta,\alpha)\to \sup_{\alpha \in \Lambda}L(\bar \theta,\alpha)$ almost surely. In the lines that follow, we make more transparent the dependence of $\hat \theta$ in the sample size $n$ and set $\hat \theta_n\stackrel{\mbox{\tiny{def}}}{=}\hat \theta$. Since $\hat \theta_n \in \Theta$ and $\Theta$ is compact, we can extract from any subsequence of $(\hat \theta_{n})_{n}$ a subsequence $(\hat \theta_{n_k})_{k}$ such that $\hat \theta_{n_k}\to z\in \Theta$ (with $n_k=n_k(\omega)$, i.e., it is almost surely defined). By continuity of the function $\theta \mapsto \sup_{\alpha \in \Lambda }L(\theta,\alpha)$, we deduce that $\sup_{\alpha \in \Lambda}L(\hat \theta_{n_k},\alpha) \to \sup_{\alpha \in \Lambda}L(z,\alpha)$, and so $\sup_{\alpha \in \Lambda}L(z,\alpha)=\sup_{\alpha \in \Lambda}L(\bar \theta,\alpha)$. Since $\bar \theta$ is unique by $(H_1)$, we have $z=\bar \theta$. In conclusion, we can extract from each subsequence of $(\hat \theta_{n})_{n}$ a subsequence that converges towards $\bar \theta$: this shows that $\hat \theta_n \to \bar \theta$ almost surely.

Finally, we have
\begin{align*}
&|L(\bar \theta,\hat \alpha)-L(\bar \theta,\bar \alpha)| \\
& \quad \leq |L(\bar \theta,\hat \alpha)-L(\hat \theta,\hat \alpha)| +|L(\hat \theta,\hat \alpha)-\hat L(\hat \theta,\hat \alpha)|+|\hat L(\hat \theta,\hat \alpha)-L(\bar \theta,\bar \alpha)|\\
& \quad = |L(\bar \theta,\hat \alpha)-L(\hat \theta,\hat \alpha)| +|L(\hat \theta,\hat \alpha)-\hat L(\hat \theta,\hat \alpha)|+|\inf_{\theta \in \Theta}\sup_{\alpha \in \Lambda}\hat L(\theta,\alpha)-\inf_{\theta \in \Theta}\sup_{\alpha \in \Lambda}L(\theta,\alpha)|\\
& \quad \leq \sup_{\alpha \in \Lambda} |L(\bar \theta,\alpha)-L(\hat \theta,\alpha)| +2 \sup_{\theta \in \Theta,\alpha \in \Lambda}|\hat L(\theta,\alpha)-L(\theta,\alpha)|.
\end{align*}
Using Assumptions $(H'_D)$ and $(H'_p)$, and the fact that $\hat \theta \to \bar \theta$ almost surely, we see that the first term above tends to zero. The second one vanishes asymptotically by Lemma \ref{differentials}, and we conclude that $L(\bar \theta,\hat \alpha)\to L(\bar \theta,\bar \alpha)$ almost surely. Since $\hat \alpha \in \Lambda$ and $\Lambda$ is compact, we may argue as in the first part of the proof and deduce from the unicity of $\bar \alpha$ that $\hat \alpha \to \bar \alpha$ almost surely.
\end{proof}
To illustrate the result of Theorem \ref{consistency}, we undertook a series of small numerical experiments with three choices for the triplet (true $p^{\star}$ + generator model $\mathscr P=\{p_{\theta}\}_{\theta \in \Theta}$ + discriminator family $\mathscr D=\{D_{\alpha}\}_{\alpha \in \Lambda}$), which we respectively call the {\bf Laplace-Gaussian}, {\bf Claw-Gaussian}, and {\bf Exponential-Uniform} model. They are summarized in Table \ref{tab:models}. We are aware that more elaborate models (involving, for example, neural networks) can be designed and implemented. However, once again, our objective is not to conduct a series of extensive simulations, but simply to illustrate our theoretical results with a few graphs to get some better intuition.
\begin{table}[!!h]
  \centering
  \begin{tabular}{|lccc|}
    \hline
    \hline
    {\bf Model} & $p^\star$ & $\mathscr P=\{p_\theta\}_{\theta \in \Theta}$ & $\mathscr D=\{D_\alpha\}_{\alpha \in \Lambda}$ \\
    \hline
    {\bf Laplace-Gaussian}
      & $\frac{1}{2 b} e^{-\frac{|x|}{b}}$
      & $\frac{1}{\sqrt{2 \pi} \theta} e^{-\frac{x^2}{2\theta^2}}$
      & $\frac{1}{1 + \frac{\alpha_1}{\alpha_0} e^{\frac{x^2}{2} (\alpha_1^{-2} - \alpha_0^{-2})}}$ \\
      & $b=1.5$
      &$\Theta= [10^{-1},10^3]$
      & $\Lambda = \Theta \times \Theta$\\
      \hline
    {\bf Claw-Gaussian}
      & $p_{\text{claw}}(x)$
      & $\frac{1}{\sqrt{2 \pi} \theta} e^{-\frac{x^2}{2\theta^2}}$
      & $\frac{1}{1 + \frac{\alpha_1}{\alpha_0} e^{\frac{x^2}{2} (\alpha_1^{-2} - \alpha_0^{-2})}}$ \\
      &
      &$\Theta= [10^{-1},10^3]$
      & $\Lambda = \Theta \times \Theta$\\
      \hline
    {\bf Exponential-Uniform}
      & $\lambda e^{-\lambda x}$
      & $\frac 1 \theta \mathbf 1_{[0, \theta]}(x)$
      & $\frac{1}{1 + \frac{\alpha_1}{\alpha_0} e^{\frac{x^2}{2} (\alpha_1^{-2} - \alpha_0^{-2})}}$ \\
      & $\lambda = 1$
      &$\Theta= [10^{-3},10^3]$
      & $\Lambda = \Theta \times \Theta$\\
      \hline
    \hline
  \end{tabular}
  \bigskip
  \caption{Triplets used in the numerical experiments.}
  \label{tab:models}
\end{table}

Figure \ref{img:true_dist} shows the densities $p^\star$. We recall that the claw density on $[0,\infty)$ takes the form
$$
    p_{\text{claw}}
    = \frac 12 \varphi(0, 1) + \frac{1}{10} \big( \varphi(-1, 0.1) + \varphi(-0.5, 0.1) + \varphi(0, 0.1) +
    \varphi(0.5, 0.1) + \varphi(1, 0.1) \big),
$$
where $\varphi(\mu, \sigma)$ is a Gaussian density with mean $\mu$ and standard deviation $\sigma$ (this density is borrowed from \citealp{De97}).
\begin{figure}[!!h]
  \center
  \includegraphics[width=0.6\textwidth]{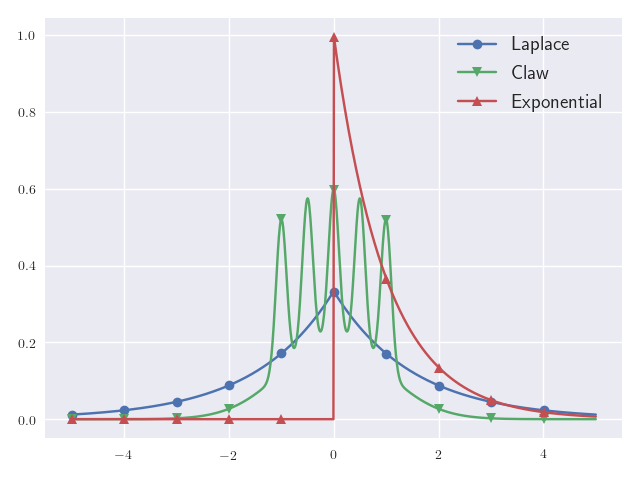}
  \caption{Probability density functions $p^\star$ used in the numerical experiments.}
  \label{img:true_dist}
\end{figure}

In the {\bf Laplace-Gaussian} and {\bf Claw-Gaussian} examples, the densities $p_\theta$ are centered Gaussian, parameterized by their standard deviation parameter $\theta$. The random variable $Z$ is uniform $[0,1]$ and the natural family of generators associated with the model $\mathscr P=\{p_{\theta}\}_{\theta \in \Theta}$ is $\mathscr G=\{G_\theta\}_{\theta \in \Theta}$, where each $G_{\theta}$ is the generalized inverse of the cumulative distribution function of $p_\theta$ (because $G_{\theta}(Z)\stackrel{\mathscr L}{=}p_{\theta}{\rm d}\mu$). The rationale behind our choice for the discriminators is based on the form of the optimal discriminator $D_{\theta}^{\star}$ described in (\ref{optimD}): starting from
$$
    D_{\theta}^{\star} = \frac{p^{\star}}{p{\star}+p_{\theta}}, \quad \theta \in \Theta,
$$
we logically consider the following ratio
 $$
    D_\alpha = \frac{p_{\alpha_1}}{p_{\alpha_1}+p_{\alpha_0}}, \quad \alpha=(\alpha_0,\alpha_1) \in \Lambda=  \Theta \times \Theta.
$$
Figure \ref{img:boxplot_lap-gauss} ({\bf Laplace-Gaussian}), Figure \ref{img:boxplot_claw-gauss} ({\bf Claw-Gaussian}), and Figure \ref{img:boxplot_exp-unif} ({\bf Exponential-Uniform}) show the boxplots of the differences $\hat \theta-\bar \theta$ over 200 repetitions, for a sample size $n$ varying from $10$ to $10\,000$. In these experiments, the parameter $\bar \theta$ is obtained by averaging the $\hat \theta$ for the largest sample size $n$. In accordance with Theorem \ref{consistency}, the size of the boxplots shrinks around 0 when $n$ increases, thus showing that the estimated parameter $\hat \theta$ is getting closer and closer to $\bar \theta$. Before analyzing at which rate this convergence occurs, we may have a look at Figure \ref{img:gen_disc}, which plots the estimated density $p_{\hat \theta}$ (for $n=10\,000$) vs.~the true density $p^\star$. It also shows the discriminator
$D_{\hat \alpha}$, together with the initial density $p_{\theta_{\text{\,init}}}$ and the initial discriminator $D_{\alpha_{\text{\,init}}}$ fed into the optimization algorithm.
We note that in the three models, $D_{\hat \alpha}$ is almost identically $1/2$, meaning that it is impossible to discriminate between the original observations and those generated by $p_{\hat \theta}$.
 \begin{figure}[!!h]
    \center
    \includegraphics[width=0.8\textwidth]{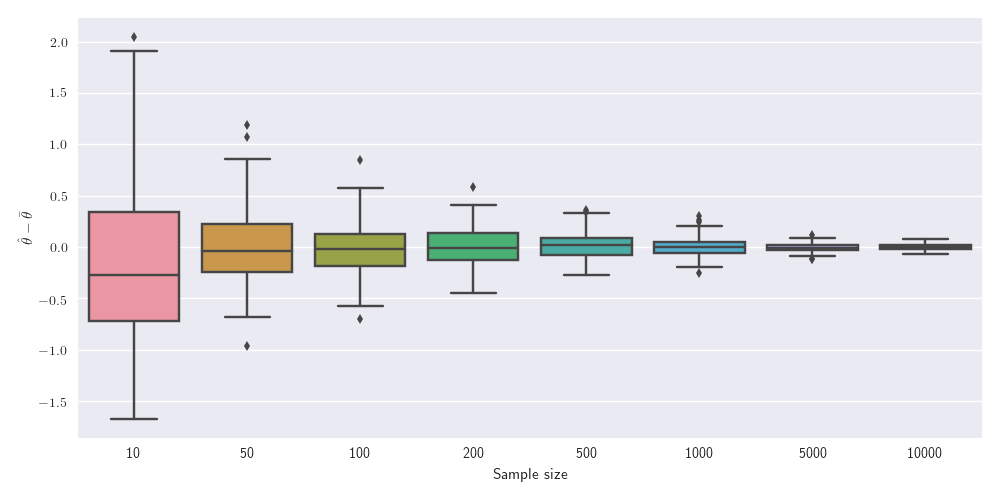}
    \caption{Boxplots of $\hat \theta - \bar \theta$ for different sample sizes ({\bf Laplace-Gaussian} model, 200 repetitions).}
    \label{img:boxplot_lap-gauss}
  \end{figure}

  \begin{figure}[!!h]
    \center
    \includegraphics[width=0.8\textwidth]{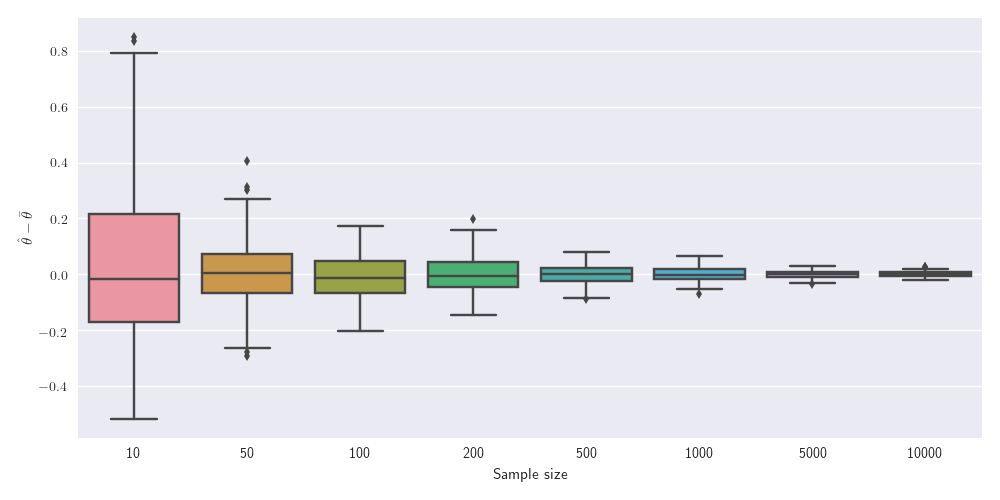}
    \caption{Boxplots of $\hat \theta - \bar \theta$ for different sample sizes ({\bf Claw-Gaussian} model, 200 repetitions).}
    \label{img:boxplot_claw-gauss}
  \end{figure}

  \begin{figure}[!!h]
    \center
    \includegraphics[width=0.8\textwidth]{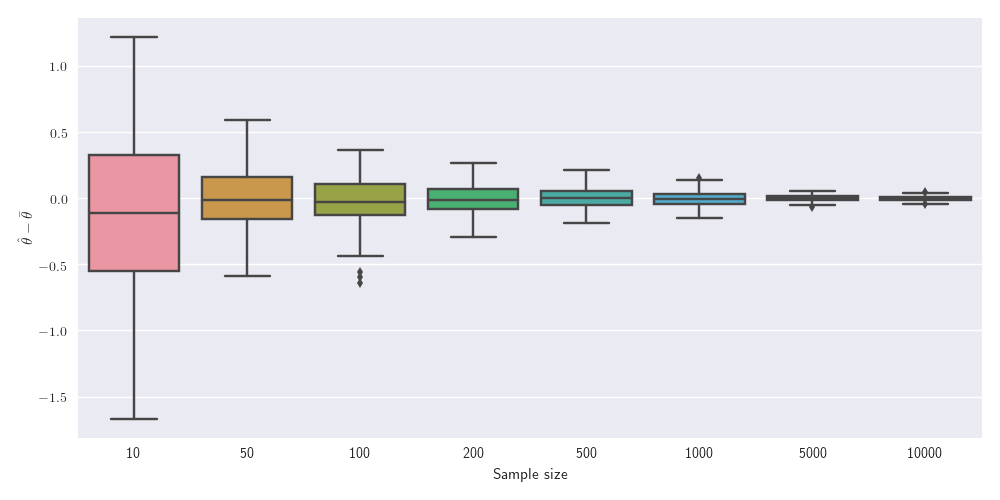}
    \caption{Boxplots of $\hat \theta - \bar \theta$ for different sample sizes ({\bf Exponential-Uniform} model, 200 repetitions).}
    \label{img:boxplot_exp-unif}
  \end{figure}

\begin{figure}[!!h]
  \center
  \includegraphics[width=0.45\textwidth]{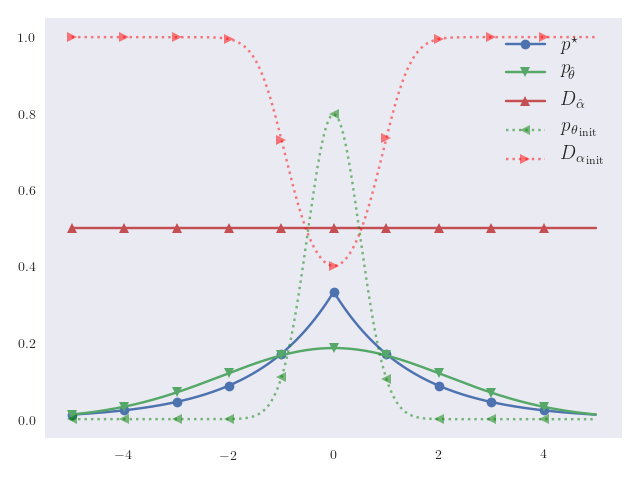}
  \includegraphics[width=0.45\textwidth]{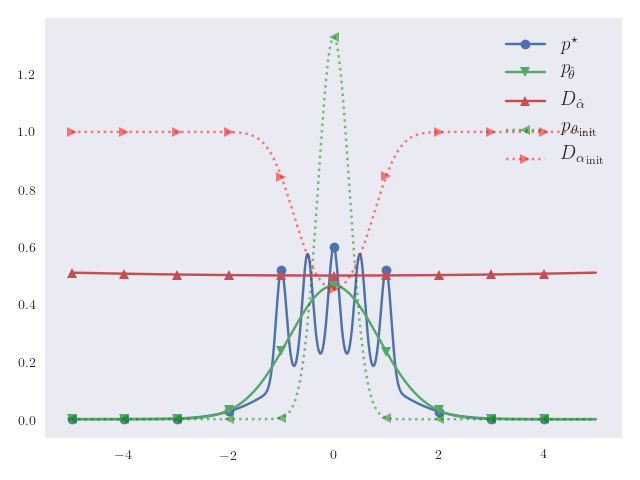}
  \includegraphics[width=0.45\textwidth]{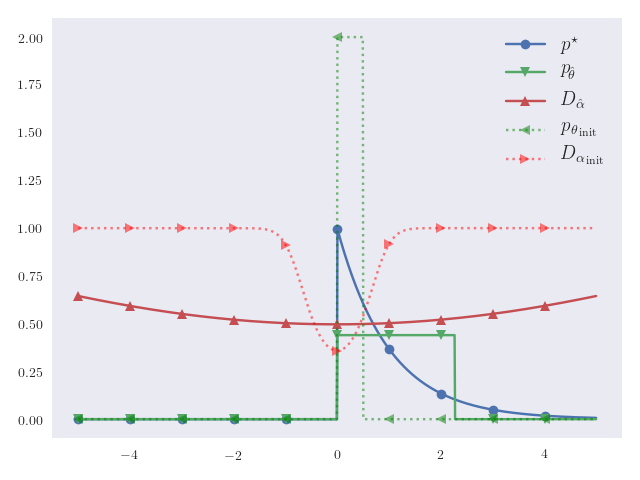}
  \caption{True density $p^\star$, estimated density $p_{\hat \theta}$, and discriminator $D_{\hat \alpha}$ for $n=10\,000$ (from left to right: {\bf Laplace-Gaussian}, {\bf Claw-Gaussian}, and {\bf Exponential-Uniform} model). Also shown are the initial density $p_{\theta_{\text{\,init}}}$ and the initial discriminator $D_{\alpha_{\text{\,init}}}$ fed into the optimization algorithm.}
  \label{img:gen_disc}
\end{figure}

In line with the above, our next step is to state a central limit theorem for $\hat \theta$. Although simple to understand, this result requires additional assumptions and some technical prerequisites.
One first needs to ensure that the function $(\theta,\alpha)\mapsto L(\theta,\alpha)$ is regular enough in a neighborhood of $(\bar \theta,\bar \alpha)$. This is captured by the following set of assumptions, which require in particular the unicity of the maximizer of the function $\alpha \mapsto L(\theta,\alpha)$ for a $\theta$ around $\bar \theta$. For a function $F:\Theta \to \mathds R$ (respectively, $G:\Theta \times \Lambda \to \mathds R$), we let $HF(\theta)$ (respectively, $H_1G(\theta,\alpha)$ and $H_2G(\theta,\alpha)$) be the Hessian matrix of the function $\theta \mapsto F(\theta)$ (respectively, $\theta \mapsto G(\theta,\alpha)$ and $\alpha \mapsto G(\theta,\alpha)$) computed at $\theta$ (respectively, at $\theta$ and $\alpha$).

{\bf Assumptions} $(H_{\text{loc}})$
\begin{enumerate}
\item[$(H_U)$] There exists a neighborhood $U$ of $\bar \theta$ and a function $\alpha: U \to \Lambda$ such that
$$\stackrel[\alpha \in \Lambda]{}{\arg\max} L(\theta,\alpha)=\{\alpha(\theta)\}, \quad \forall \theta \in U.$$
\item[$(H_V)$] The Hessian matrix $HV(\bar \theta)$ is invertible, where $V(\theta)\stackrel{\mbox{\tiny{def}}}{=}L(\theta,\alpha(\theta))$.
\item[$(H_H)$] The Hessian matrix $H_2L(\bar \theta,\bar \alpha)$ is invertible.
\end{enumerate}

We stress that under Assumption $(H_U)$, there is for each $\theta \in U$ a unique $\alpha(\theta) \in \Lambda$ such that $L(\theta,\alpha(\theta))=\sup_{\alpha \in \Lambda}L(\theta,\alpha)$. We also note that $\alpha(\bar \theta)=\bar \alpha$ under $(H_1)$. We still need some notation before we state the central limit theorem. For a function $f(\theta,\alpha)$, $\nabla_1f(\theta,\alpha)$ (respectively, $\nabla_2f(\theta,\alpha)$) means the gradient of the function $\theta \mapsto f(\theta,\alpha)$ (respectively, the function $\alpha \mapsto f(\theta,\alpha)$) computed at $\theta$ (respectively, at $\alpha)$. For a function $g(t)$, $J(g)_{t}$ is the Jacobian matrix of $g$ computed at $t$. Observe that by the envelope theorem,
$$HV(\bar \theta)=H_1L(\bar \theta,\bar \alpha)+J(\nabla_1L(\bar \theta, \cdot))_{\bar \alpha}J(\alpha)_{\bar \theta},$$
where, by the chain rule,
\begin{equation*}
J(\alpha)_{\bar \theta}=-H_2L(\bar \theta,\bar \alpha)^{-1}J(\nabla_2L(\cdot, \bar \alpha))_{\bar \theta}.
\end{equation*}
Therefore, in Assumption$(H_V)$, the Hessian matrix $HV(\bar \theta)$ can be computed with the sole knowledge of $L$. Finally, we let
$$\ell_1(\theta,\alpha)=\ln D_{\alpha}(X_1)+\ln(1-D_{\alpha}\circ G_{\theta}(Z_1)),$$
and denote by $\stackrel{\mathscr L}{\to}$ the convergence in distribution.
\begin{theo}
\label{clt}
Under Assumptions $(H'_{\emph{reg}})$, $(H_1)$, and $(H_{\emph{loc}})$, one has
$$\sqrt n (\hat \theta-\bar \theta)\stackrel{\mathscr L}{\to} Z,$$
where $Z$ is a Gaussian random variable with mean $0$ and variance
$${\mathbf V}=\emph{Var} \big[-HV(\bar \theta)^{-1}\nabla_1 \ell_1(\bar \theta,\bar \alpha)+HV(\bar \theta)^{-1}J(\nabla_1L(\bar \theta,\cdot))_{\bar \alpha}H_2L(\bar \theta,\bar \alpha)^{-1}\nabla_2\ell_1(\bar \theta,\bar \alpha)\big].$$
\end{theo}
We note that the expression of the variance is relatively complex and, unfortunately, that it cannot be simplified, even for a dimension of the parameter equal to $1$. Nevertheless, the take-home message is that the estimator $\hat \theta$ is asymptotically normal, with a convergence rate of $\sqrt n$. This is illustrated in Figures \ref{img:hist_lap-gauss}, \ref{img:hist_claw-gauss}, and \ref{img:hist_exp-unif}, which respectively show the histograms and kernel estimates of the distribution of $\sqrt n (\hat \theta - \bar \theta)$ for the {\bf Laplace-Gaussian}, the {\bf Claw-Gaussian}, and the {\bf Exponential-Uniform} model in function of the sample size $n$ (200 repetitions).
\begin{figure}[!!h]
    \center
    \includegraphics[width=0.8\textwidth]{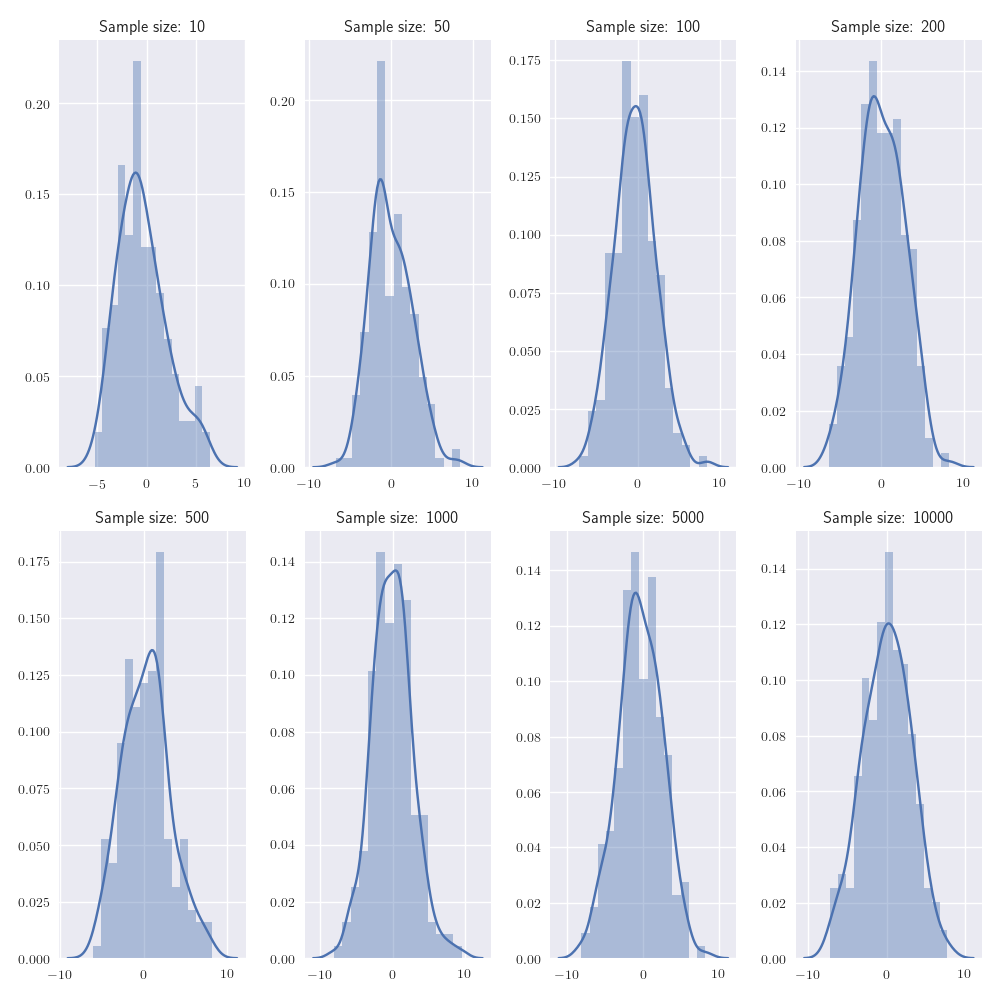}
    \caption{Histograms and kernel estimates (continuous line) of the distribution of $\sqrt n (\hat \theta - \bar \theta)$ for different sample sizes $n$ ({\bf Laplace-Gaussian} model, 200 repetitions).}
    \label{img:hist_lap-gauss}
  \end{figure}

  \begin{figure}[!!h]
    \center
    \includegraphics[width=0.8\textwidth]{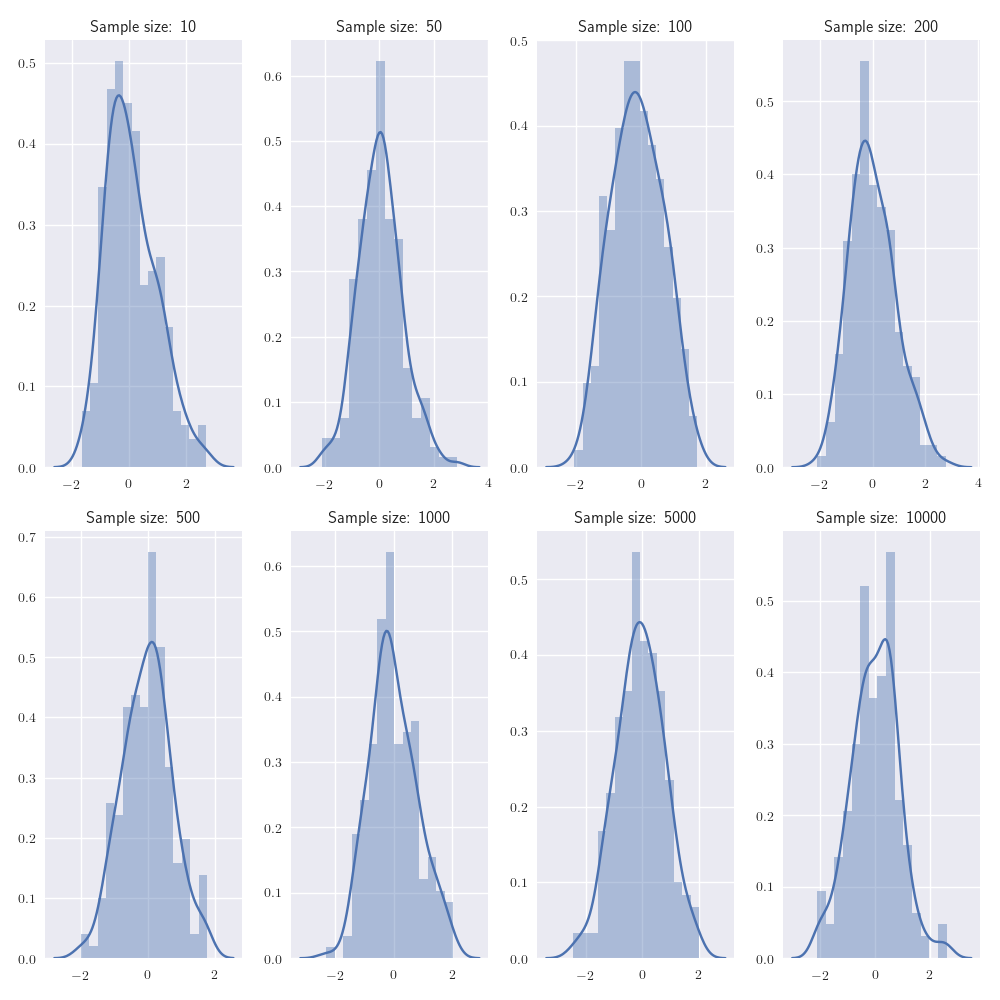}
    \caption{Histograms and kernel estimates (continuous line) of the distribution of $\sqrt n (\hat \theta - \bar \theta)$ for different sample sizes $n$ ({\bf Claw-Gaussian} model, 200 repetitions).}
    \label{img:hist_claw-gauss}
  \end{figure}

  \begin{figure}[!!h]
    \center
    \includegraphics[width=0.8\textwidth]{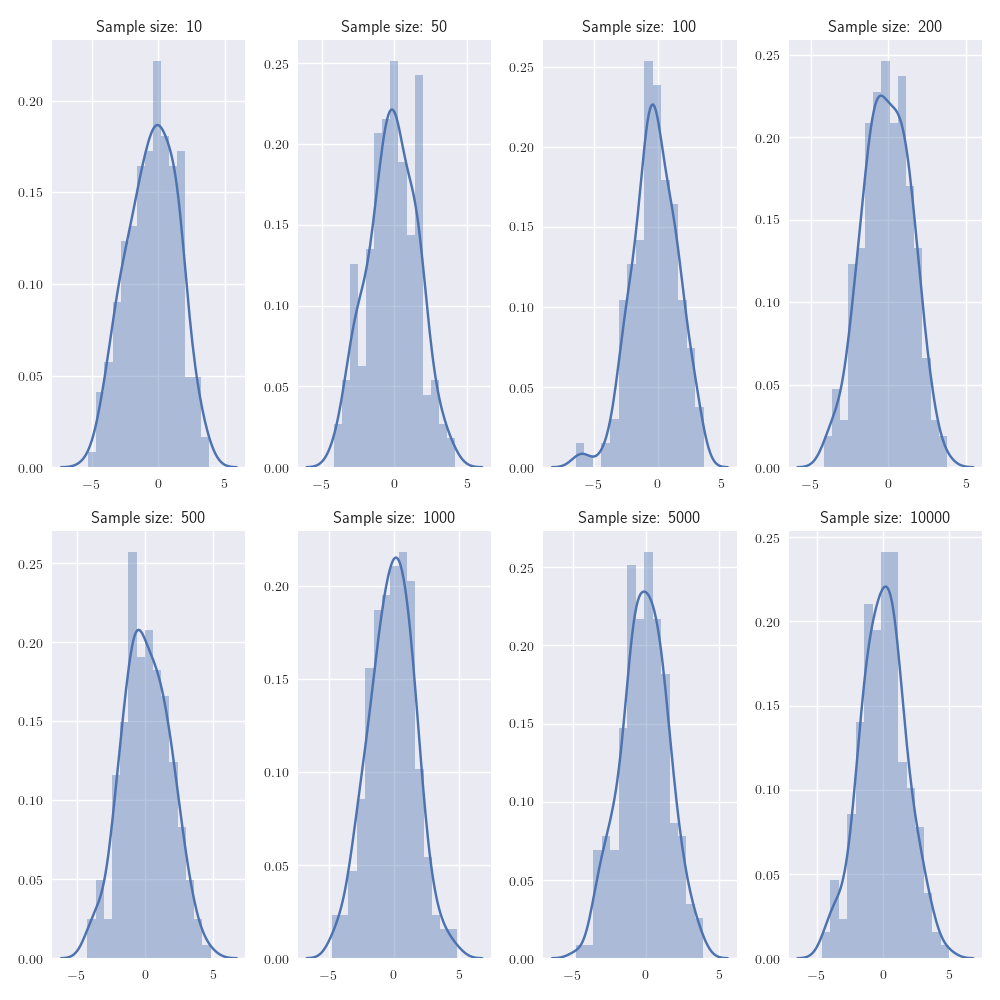}
    \caption{Histograms and kernel estimates (continuous line) of the distribution of $\sqrt n (\hat \theta - \bar \theta)$ for different sample sizes $n$ ({\bf Exponential-Uniform} model, 200 repetitions).}
    \label{img:hist_exp-unif}
  \end{figure}

\begin{proof}
By technical Lemma \ref{openset}, we can find under Assumptions $(H'_{\text{reg}})$ and $(H_1)$ an open set $V\subset U \subset \Theta^{\circ}$ containing $\bar \theta$ such that, for all $\theta \in V$, $\alpha(\theta) \in  \Lambda^{\circ}$. In the sequel, to lighten the notation, we assume without loss of generality that $V=U$. Thus, for all $\theta \in U$, we have $\alpha(\theta) \in \Lambda^{\circ}$ and $L(\theta,\alpha(\theta))=\sup_{\alpha \in \Lambda} L(\theta,\alpha)$ (with $\alpha(\bar \theta)=\bar \alpha$ by $(H_1)$).  Accordingly, $\nabla_2L(\theta,\alpha(\theta))=0$, $\forall \theta \in U$. Also, since $H_2L(\bar \theta,\bar \alpha)$ is invertible by $(H_H)$ and since the function $(\theta,\alpha)\mapsto H_2L(\theta,\alpha)$ is continuous, there exists an open set $U' \subset U$ such that $H_2L(\theta,\alpha)$ is invertible as soon as $(\theta,\alpha) \in (U',\alpha(U'))$. Without loss of generality, we assume that $U'=U$. Thus, by the chain rule, the function $\alpha$ is of class $C^2$ in a neighborhood $U'\subset U$ of $\bar \theta$, say $U'=U$, with Jacobian matrix given by
$$J(\alpha)_{\theta}=-H_2L(\theta,\alpha(\theta))^{-1}J\big(\nabla_2L(\cdot,\alpha(\theta))\big)_{\theta}, \quad \forall \theta \in U.$$
We note that $H_2L(\theta,\alpha(\theta))^{-1}$ is of format $q\times q$ and $J(\nabla_2L(\cdot,\alpha(\theta)))_{\theta}$ of format $q\times p$.

Now, for each $\theta \in U$, we let $\hat \alpha(\theta)$ be such that $\hat L(\theta,\hat \alpha(\theta))=\sup_{\alpha \in \Lambda} \hat L(\theta,\alpha)$. Clearly,
\begin{align*}
|L(\theta,\hat \alpha(\theta))-L(\theta,\alpha(\theta))|  & \leq |L(\theta,\hat \alpha(\theta))-\hat L(\theta,\hat\alpha(\theta))| +|\hat L(\theta,\hat \alpha(\theta))-L(\theta,\alpha(\theta))|\\
& \leq \sup_{\alpha \in \Lambda} |L(\theta,\alpha)-\hat L(\theta,\alpha)|+|\sup_{\alpha \in \Lambda}\hat L(\theta,\alpha)-\sup_{\alpha \in \Lambda}L(\theta,\alpha)|\\
& \leq 2 \sup_{\alpha \in \Lambda} |\hat L(\theta,\alpha)-L(\theta,\alpha)|.
\end{align*}
Therefore, by Lemma \ref{differentials}, $\sup_{\theta \in U}|L(\theta,\hat\alpha(\theta))-L(\theta,\alpha(\theta))|\to 0$ almost surely. The event on which this convergence holds does not depend upon $\theta \in U$, and, arguing as in the proof of Theorem \ref{consistency}, we deduce that under $(H_1)$, $\mathds P(\hat \alpha(\theta)\to \alpha(\theta)\,\forall \theta \in U)=1$. Since $\alpha(\theta) \in \Lambda^{\circ}$ for all $\theta \in U$, we also have $\mathds P(\hat \alpha(\theta)\in \Lambda^{\circ}\, \forall \theta \in U) \to 1$ as $n \to \infty$. Thus, in the sequel, it will be assumed without loss of generality that, for all $\theta \in U$, $\hat \alpha(\theta) \in \Lambda^{\circ}$.

Still by Lemma \ref{differentials}, $\sup_{\theta \in \Theta, \alpha \in \Lambda}\|H_2\hat L(\theta,\alpha)- H_2L(\theta,\alpha)\|\to 0$ almost surely. Since $H_2L(\theta,\alpha)$ is invertible on $U\times \alpha(U)$, we have
$$\mathds P\big(H_2\hat L(\theta,\alpha)\mbox{ invertible }\forall (\theta,\alpha) \in U\times \alpha(U)\big)\to 1.$$
Thus, we may and will assume that $H_2\hat L(\theta,\alpha)$ is invertible for all $(\theta,\alpha)\in U\times \alpha(U)$.

Next, since $\hat \alpha(\theta)  \in \Lambda^{\circ}$ for all $\theta \in U$, one has $\nabla_2\hat L(\theta,\hat \alpha(\theta))=0$. Therefore, by the chain rule, $\hat \alpha$ is of class $C^2$ on $U$, with Jacobian matrix
$$J(\hat \alpha)_{\theta}=-H_2\hat L(\theta,\hat \alpha(\theta))^{-1}J\big(\nabla_2\hat L(\cdot,\hat \alpha(\theta))\big)_{\theta}, \quad \forall \theta \in U.$$

Let $\hat V(\theta)\stackrel{\mbox{\tiny{def}}}{=}\hat L(\theta,\hat \alpha(\theta))=\sup_{\alpha \in \Lambda} \hat L(\theta,\alpha)$. By the envelope theorem, $\hat V$ is of class $C^2$, $\nabla \hat V(\theta)=\nabla_1 \hat L(\theta,\hat \alpha(\theta))$, and $H\hat V(\theta)=H_1\hat L(\theta,\hat \alpha(\theta))+J(\nabla_1\hat L(\theta,\cdot))_{\hat \alpha(\theta)}J(\hat \alpha)_{\theta}$. Recall that $\hat \theta \to \bar \theta$ almost surely by Theorem \ref{consistency}, so that we may assume that $\hat \theta \in \Theta^{\circ}$ by $(H_1)$. Moreover, we can also assume that $\hat \theta+t(\hat \theta-\bar \theta) \in U$, $\forall t \in [0,1]$. Thus, by a Taylor series expansion with integral remainder, we have
\begin{equation}
\label{TE1}
0=\nabla \hat V(\hat \theta)=\nabla \hat V(\bar \theta)+\int_0^1H\hat V(\hat \theta+t(\hat \theta-\bar \theta)){\rm d}t(
\hat \theta-\bar \theta).
\end{equation}
Since $\hat \alpha(\bar \theta) \in \Lambda^{\circ}$ and $\hat L(\bar \theta,\hat \alpha(\bar \theta))=\sup_{\alpha \in \Lambda} \hat L(\bar \theta,\alpha)$, one has $\nabla_2 \hat L(\bar \theta,\hat \alpha(\bar \theta))=0$. Thus,
\begin{align*}
0&=\nabla_2\hat L(\bar \theta,\hat \alpha(\bar \theta))\\
&=\nabla_2 \hat L(\bar \theta,\alpha(\bar \theta))+\int_0^1 H_2 \hat L\big(\bar \theta,\alpha(\bar \theta)+t(\hat \alpha(\bar \theta)-\alpha(\bar \theta))\big){\rm d}t (\hat \alpha(\bar \theta)-\alpha(\bar \theta)).
\end{align*}
By Lemma \ref{differentials}, since $\hat \alpha(\bar \theta) \to \alpha(\bar \theta)$ almost surely, we have
$$\hat I_1\stackrel{\mbox{\tiny{def}}}{=}\int_0^1 H_2\hat L\big(\bar \theta,\alpha(\bar \theta)+t(\hat \alpha(\bar \theta)-\alpha(\bar \theta))\big){\rm d}t\to H_2 L(\bar \theta,\bar \alpha)\quad \mbox{almost surely}.$$
Because $H_2L(\bar \theta,\bar \alpha)$ is invertible, $\mathds P(\hat I_1\mbox{ invertible})\to 1$ as $n \to \infty$. Therefore, we may assume, without loss of generality, that $\hat I_1$ is invertible. Hence,
\begin{equation}
\label{TE2}
\hat \alpha(\bar \theta)-\alpha(\bar \theta)=-\hat I_1^{-1}\nabla_2\hat L(\bar \theta,\alpha(\bar \theta)).
\end{equation}
Furthermore,
$$\nabla \hat V(\bar \theta)=\nabla_1\hat L (\bar \theta,\hat \alpha(\bar \theta))=\nabla_1\hat L(\bar \theta,\alpha(\bar \theta))+\hat I_2(\hat \alpha(\bar \theta)-\alpha(\bar \theta)),$$
where
$$\hat I_2\stackrel{\mbox{\tiny{def}}}{=}\int_0^1{J(\nabla_1 \hat L(\bar \theta,\cdot))}_{\alpha(\bar \theta)+t(\hat \alpha(\bar \theta)-\alpha(\bar \theta))}{\rm d}t.$$
By Lemma \ref{differentials}, $\hat I_2 \to J(\nabla_1 L(\bar \theta,\cdot))_{\alpha(\bar \theta)}$ almost surely. Combining (\ref{TE1}) and (\ref{TE2}), we obtain
$$0=\nabla_1 \hat L(\bar \theta,\alpha(\bar \theta))-\hat I_2\hat I_1^{-1}\nabla_2 \hat L (\bar \theta,\alpha(\bar \theta))+\hat I_3(\hat \theta-\bar \theta),$$
where
$$\hat I_3\stackrel{\mbox{\tiny{def}}}{=}\int_0^1 H\hat V(\hat \theta+t(\hat \theta-\bar \theta)){\rm d}t.$$
By technical Lemma \ref{lem-tech1}, we have $\hat I_3 \to HV(\bar \theta)$ almost surely. So, by $(H_V)$, it can be assumed that $\hat I_3$ is invertible. Consequently,
$$\hat \theta-\bar \theta=-\hat I_3^{-1}\nabla_1 \hat L(\bar \theta,\alpha(\bar \theta))+\hat I_3^{-1}\hat I_2\hat I_1^{-1}\nabla_2 \hat L(\bar \theta,\alpha(\bar \theta)),$$
or, equivalently, since $\alpha(\bar \theta)=\bar \alpha$,
$$\hat \theta-\bar \theta=-\hat I_3^{-1}\nabla_1 \hat L(\bar \theta,\bar \alpha)+\hat I_3^{-1}\hat I_2\hat I_1^{-1}\nabla_2 \hat L(\bar \theta,\bar \alpha).$$
Using Lemma \ref{differentials}, we conclude that $\sqrt n(\hat \theta-\bar \theta)$ has the same limit distribution as
$$S_n\stackrel{\mbox{\tiny{def}}}{=}-\sqrt n HV(\bar \theta)^{-1}\nabla_1 \hat L(\bar \theta,\bar \alpha)+\sqrt nHV(\bar \theta)^{-1}J(\nabla_1L(\bar \theta,\cdot))_{\bar \alpha}H_2L(\bar \theta,\bar \alpha)^{-1}\nabla_2\hat L(\bar \theta,\bar \alpha).$$

Let
$$\ell_i(\theta,\alpha)=\ln D_{\alpha}(X_i)+\ln(1-D_{\alpha}\circ G_{\theta}(Z_i)), \quad 1 \leq i \leq n.$$
With this notation, we have
$$S_n=\frac{1}{\sqrt n} \sum_{i=1}^n \Big(-HV(\bar \theta)^{-1}\nabla_1 \ell_i(\bar \theta,\bar \alpha)+HV(\bar \theta)^{-1}J(\nabla_1L(\bar \theta,\cdot))_{\bar \alpha}H_2L(\bar \theta,\bar \alpha)^{-1}\nabla_2\ell_i(\bar \theta,\bar \alpha)\Big).$$
One has $\nabla V(\bar \theta)=0$, since $V(\bar \theta)=\inf_{\theta \in \Theta} V(\theta)$ and $\bar \theta \in \Theta^{\circ}$. Therefore, under $(H'_{\text{reg}})$, $\mathds E\nabla_1 \ell_i(\bar \theta,\bar \alpha)=\nabla_1\mathds E\ell_i(\bar \theta,\bar \alpha)=\nabla_1 L(\bar \theta,\bar \alpha)=\nabla V(\bar \theta)=0$.
Similarly, $\mathds E\nabla_2 \ell_i(\bar \theta,\bar \alpha)=\nabla_2\mathds E\ell_i(\bar \theta,\bar \alpha)=\nabla_2 L(\bar \theta,\bar \alpha)=0$, since $L(\bar \theta,\bar \alpha)=\sup_{\alpha \in \Lambda} L(\bar \theta,\alpha)$ and $\bar \alpha \in \Lambda^{\circ}$.
Using the central limit theorem, we conclude that
$$\sqrt n (\hat \theta-\bar \theta)\stackrel{\mathscr L}{\to} Z,$$
where $Z$ is a Gaussian random variable with mean $0$ and variance
$${\mathbf V}=\mbox{Var} \big[-HV(\bar \theta)^{-1}\nabla_1 \ell_1(\bar \theta,\bar \alpha)+HV(\bar \theta)^{-1}J(\nabla_1L(\bar \theta,\cdot))_{\bar \alpha}H_2L(\bar \theta,\bar \alpha)^{-1}\nabla_2\ell_1(\bar \theta,\bar \alpha)\big].$$
\end{proof}
\section{Technical results}
\label{STL}
\subsection{Proof of Theorem \ref{theorem-encadrement}}
Choose $\varepsilon \in (0,\underline t)$ and $D\in \mathscr D$, a $\bar \theta$-admissible discriminator, such that $\|D-D_{\bar \theta}^{\star}\|_{\infty}\leq \varepsilon$. Observe that
\begin{align}
L(\bar \theta,D)&=\int \ln(D)p^{\star}{\rm d}\mu+\int \ln(1-D)p_{\bar \theta}{\rm d}\mu\nonumber\\
&=\int \ln \Big(\frac{D}{D_{\bar \theta}^{\star}}\Big)p^{\star}{\rm d}\mu+\ln \Big(\frac{1-D}{1-D_{\bar \theta}^{\star}}\Big)p_{\bar \theta}{\rm d}\mu+2\JS{p^{\star}}{p_{\bar \theta}}-\ln 4.\label{zero}
\end{align}
Clearly,
\begin{align*}
\int \ln \Big(\frac{D}{D^{\star}_{\bar \theta}}\Big)p^{\star}{\rm d}\mu&=\int \ln \Big(1+\Big[\frac{D}{D^{\star}_{\bar \theta}}-1\Big]\Big)p^{\star}{\rm d}\mu\\
&=\int \ln \Big(1+\frac{\gamma_{\bar \theta}}{D^{\star}_{\bar \theta}}\Big)p^{\star}{\rm d}\mu,
\end{align*}
where $\gamma_{\bar \theta}=D-D_{\bar \theta}^{\star}$. By a Taylor series expansion with remainder, we may write
$$\ln \Big(1+\frac{\gamma_{\bar \theta}}{D_{\bar \theta}^{\star}}\Big)=\frac{\gamma_{\bar \theta}}{D^{\star}_{\bar \theta}}-\frac{1}{2}\Big(\frac{\gamma_{\bar \theta}}{D_{\bar \theta}^{\star}}\Big)^2+\frac{1}{3}\int_0^{\gamma_{\bar \theta}/D^{\star}_{\bar \theta}}\frac{1}{(1+u)^3}\Big(\frac{\gamma_{\bar \theta}}{D^{\star}_{\bar \theta}}-u\Big)^2{\rm d}u.$$
Whenever $\gamma_{\bar \theta}\leq 0$ (worst case), we have
$$\int_0^{\gamma_{\bar \theta}/D^{\star}_{\bar \theta}}\frac{1}{(1+u)^3}\Big(\frac{\gamma_{\bar \theta}}{D^{\star}_{\bar \theta}}-u\Big)^2{\rm d}u=-\int_{\gamma_{\bar \theta}/D^{\star}_{\bar \theta}}^0\frac{1}{(1+u)^3}\Big(\frac{\gamma_{\bar \theta}}{D^{\star}_{\bar \theta}}-u\Big)^2{\rm d}u.$$
Observe that, for $\gamma_{\bar \theta}/D^{\star}_{\bar \theta}\leq u \leq 0$, since $\|\gamma_{\bar \theta}\|_{\infty} \leq \varepsilon$ by assumption and $D^{\star}_{\bar \theta}\geq \underline t$ by $(H_0)$,
$$1+u\geq 1+\frac{\gamma_{\bar \theta}}{D^{\star}_{\bar \theta}}\geq 1-\frac{\varepsilon}{D^{\star}_{\bar \theta}}\geq 1-\frac{\varepsilon}{\underline t}>0.$$
Thus,
\begin{equation}
\label{one}
\int \ln \Big(\frac{D}{D^{\star}_{\bar \theta}}\Big)p^{\star}{\rm d}\mu\geq \int \Big (\frac{\gamma_{\bar \theta}}{D^{\star}_{\bar \theta}}-\frac{1}{2}\Big(\frac{\gamma_{\bar \theta}}{D^{\star}_{\bar \theta}}\Big)^2-\frac{1}{9}\Big(\frac{|\gamma_{\bar \theta}|}{D^{\star}_{\bar \theta}}\Big)^3 \frac{1}{(1-\varepsilon/{\underline t})^3}\Big)p^{\star}{\rm d}\mu.
\end{equation}
Similarly, we have
\begin{align*}
\int \ln \Big(\frac{1-D}{1-D^{\star}_{\bar \theta}}\Big)p_{\bar \theta}{\rm d}\mu&=\int \ln \Big(1+\Big[\frac{1-D}{1-D^{\star}_{\bar \theta}}-1\Big]\Big)p_{\bar \theta}{\rm d}\mu\\
&=\int \ln \Big(1-\frac{\gamma_{\bar \theta}}{1-D^{\star}_{\bar \theta}}\Big)p_{\bar \theta}{\rm d}\mu.
\end{align*}
By a Taylor series with remainder,
$$\ln \Big(1-\frac{\gamma_{\bar \theta}}{1-D^{\star}_{\bar \theta}}\Big)=-\frac{\gamma_{\bar \theta}}{1-D^{\star}_{\bar \theta}}-\frac{1}{2}\Big(\frac{\gamma_{\bar \theta}}{1-D^{\star}_{\bar \theta}}\Big)^2+\frac{1}{3}\int_0^{-\gamma_{\bar \theta}/(1-D^{\star}_{\bar \theta})}\frac{1}{(1+u)^3}\Big(\frac{\gamma_{\bar \theta}}{1-D^{\star}_{\bar \theta}}+u\Big)^2{\rm d}u.$$
Whenever $\gamma_{\bar \theta}\geq 0$ (worst case), we have
$$\int_0^{-\gamma_{\bar \theta}/(1-D^{\star}_{\bar \theta})}\frac{1}{(1+u)^3}\Big(\frac{\gamma_{\bar \theta}}{1-D^{\star}_{\bar \theta}}+u\Big)^2{\rm d}u=-\int_{-\gamma_{\bar \theta}/(1-D^{\star}_{\bar \theta})}^0\frac{1}{(1+u)^3}\Big(\frac{\gamma_{\bar \theta}}{1-D^{\star}_{\bar \theta}}+u\Big)^2{\rm d}u.$$
But, for $-\frac{\gamma_{\bar \theta}}{1-D^{\star}_{\bar \theta}}\leq u \leq 0$,
$$1+u\geq 1-\frac{\gamma_{\bar \theta}}{1-D^{\star}_{\bar \theta}}\geq 1-\frac{\varepsilon}{1-D^{\star}_{\bar \theta}}\geq 1-\frac{\varepsilon}{\underline t}> 0.$$
Thus, we obtain
\begin{equation}
\label{two}
\int \ln \Big(\frac{1-D}{1-D^{\star}_{\bar \theta}}\Big)p_{\bar \theta}{\rm d}\mu \geq \int \Big( -\frac{\gamma_{\bar \theta}}{1-D^{\star}_{\bar \theta}}-\frac{1}{2}\Big(\frac{\gamma_{\bar \theta}}{1-D^{\star}_{\bar \theta}}\Big)^2-\frac{1}{9}\Big(\frac{|\gamma_{\bar \theta}|}{1-D^{\star}_{\bar \theta}}\Big)^3\frac{1}{(1-\varepsilon/{\underline t})^3}\Big)p_{\bar \theta}{\rm d}\mu.
\end{equation}
Letting
$$\tau=\frac{1}{(1-\varepsilon/{\underline t})^3},$$
and combining (\ref{zero}), (\ref{one}), and (\ref{two}), we are led to
\begin{align*}
L(\bar \theta,D) &\geq \int \Big (\frac{\gamma_{\bar \theta}}{D^{\star}_{\bar \theta}}-\frac{1}{2}\Big(\frac{\gamma_{\bar \theta}}{D^{\star}_{\bar \theta}}\Big)^2-\frac{1}{9}\Big(\frac{|\gamma_{\bar \theta}|}{D^{\star}_{\bar \theta}}\Big)^3 \frac{1}{\tau}\Big)p^{\star}{\rm d}\mu\\
& \quad +\int \Big( -\frac{\gamma_{\bar \theta}}{1-D^{\star}_{\bar \theta}}-\frac{1}{2}\Big(\frac{\gamma_{\bar \theta}}{1-D^{\star}_{\bar \theta}}\Big)^2-\frac{1}{9}\Big(\frac{|\gamma_{\bar \theta}|}{1-D^{\star}_{\bar \theta}}\Big)^3\frac{1}{\tau}\Big)p_{\bar \theta}{\rm d}\mu\\
& \quad +2 \JS{p^{\star}}{p_{\bar \theta}}-\ln 4 \\
& \geq -\frac{\varepsilon^2}{2}\int \frac{p^{\star}}{D_{\bar \theta}^{\star 2}}{\rm d}\mu-\frac{\varepsilon^2}{2}\int \frac{p_{\bar \theta}}{(1-D^{\star}_{\bar \theta})^2}{\rm d}\mu-\frac{\varepsilon^3}{9\tau}\int \Big(\frac{p^{\star}}{D_{\bar \theta}^{\star 3}}+\frac{p_{\bar \theta}}{(1-D^{\star}_{\bar \theta})^3}\Big){\rm d}\mu \nonumber\\
& \quad +2 \JS{p^{\star}}{p_{\bar \theta}}-\ln 4 \\
& =-\frac{\varepsilon ^2}{2} \Big ( \int \frac{(p^{\star}+p_{\bar \theta})^2}{p^{\star}}{\rm d}\mu+\int \frac{(p^{\star}+p_{\bar \theta})^2}{p_{\bar \theta}}{\rm d}\mu\Big)\\
& \qquad -\frac{\varepsilon^3}{9\tau}\int \Big(\frac{(p^{\star}+p_{\bar \theta})^3}{p^{\star 2}}+\frac{(p^{\star}+p_{\bar \theta})^3}{p_{\bar \theta}^{2}}\Big){\rm d}\mu +2 \JS{p^{\star}}{p_{\bar \theta}}-\ln 4.
\end{align*}
Using $(H_0)$, we conclude that there exists a constant $c>0$ (depending only upon $\underline t$) such that
$$L(\bar \theta,D)\geq -c\varepsilon^2-\frac{c}{\tau}\varepsilon^3+2\JS{p^{\star}}{p_{\bar \theta}}-\ln 4,$$
i.e.,
$$2\JS{p^{\star}}{p_{\bar \theta}}\leq c\varepsilon^2+\frac{c}{\tau}\varepsilon^3+L(\bar \theta,D)+\ln 4.$$
But
\begin{align*}
L(\bar \theta,D) &\leq \sup_{D \in \mathscr D}L(\bar \theta,D)\\
& \leq \sup_{D \in \mathscr D}L(\theta^{\star},D)\\
& \quad \mbox{(by definition of $\bar \theta$)}\\
&\leq \sup_{D \in \mathscr D_{\infty}}L(\theta^{\star},D)\\
&= L(\theta^{\star},D^{\star}_{\theta^{\star}})=2\JS{p^{\star}}{p_{\theta^{\star}}}-\ln 4.
\end{align*}
Thus,
$$2\JS{p^{\star}}{p_{\bar \theta}}\leq c\varepsilon^2+\frac{c}{\tau}\varepsilon^3+2\JS{p^{\star}}{p_{\theta^{\star}}}.$$
This shows the right-hand side of inequality (\ref{encadrement}). To prove the left-hand side, just note that by inequality (\ref{PA}),
$$\JS{p^{\star}}{p_{\theta^{\star}}} \leq \JS{p^{\star}}{p_{\bar \theta}}.$$
\subsection{Proof of Lemma \ref{differentials}}
To simplify the notation, we set
$$\Delta=\frac{\partial ^{a+b+c+d}}{\partial \theta_i^a \partial \theta_j^b \partial \alpha_{\ell}^c \partial\alpha_m^d}.$$
Using McDiarmid's inequality \citep{Mc89}, we see that there exists a constant $c>0$ such that, for all $\varepsilon>0$,
$$\mathds P \Big(\Big| \sup_{\theta \in \Theta, \alpha \in \Lambda} |\Delta\hat L(\theta,\alpha)-\Delta L(\theta,\alpha)|-\mathds E\sup_{\theta \in \Theta, \alpha \in \Lambda} |\Delta\hat L(\theta,\alpha)-\Delta L(\theta,\alpha)| \Big|\geq \varepsilon\Big)\leq
2e^{-cn \varepsilon^2}.$$
Therefore, by the Borel-Cantelli lemma,
\begin{equation}
\label{tic}
\sup_{\theta \in \Theta, \alpha \in \Lambda} |\Delta\hat L(\theta,\alpha)-\Delta L(\theta,\alpha)|-\mathds E\sup_{\theta \in \Theta, \alpha \in \Lambda} |\Delta\hat L(\theta,\alpha)-\Delta L(\theta,\alpha)| \to 0 \quad \mbox{almost surely}.
\end{equation}
It is also easy to verify that under Assumptions $(H'_{\text{reg}})$, the process $(\Delta \hat L(\theta,\alpha)-\Delta L(\theta,\alpha))_{\theta \in \Theta,\alpha \in \Lambda}$ is subgaussian. Thus, as in the proof of Theorem \ref{rate}, we obtain via Dudley's inequality that
\begin{equation}
\label{tac}
\mathds E\sup_{\theta \in \Theta, \alpha \in \Lambda} |\Delta \hat L(\theta,\alpha)-\Delta L(\theta,\alpha)|={\rm O}\Big(\frac{1}{\sqrt n}\Big),
\end{equation}
since $\mathds E \Delta\hat L(\theta,\alpha)=\Delta L(\theta,\alpha)$. The result follows by combining (\ref{tic}) and (\ref{tac}).
\subsection{Some technical lemmas}
\begin{lem}
\label{openset}
Under Assumptions $(H'_{\emph{reg}})$ and $(H_1)$, there exists an open set $V \subset \Theta^{\circ}$ containing $\bar \theta$ such that, for all $\theta \in V$, ${\arg \max}_{\alpha \in \Lambda}L(\theta,\alpha) \cap  \Lambda^{\circ} \neq \emptyset$.
\end{lem}
\begin{proof}
Assume that the statement is not true. Then there exists a sequence $(\theta_{k})_k\subset \Theta$ such that $\theta_k \to \bar \theta$ and, for all $k$, $\alpha_k \in \partial \Lambda$, where $\alpha_k \in {\arg \max}_{\alpha \in \Lambda} L(\theta_k,\alpha)$. Thus, since $\Lambda$ is compact, even if this means extracting a subsequence, one has $\alpha_k \to z \in \partial \Lambda$ as $k \to \infty$. By the continuity of $L$, $L(\bar \theta,\alpha_k) \to L(\bar \theta,z)$. But
\begin{align*}
|L(\bar \theta,\alpha_k)-L(\bar \theta,\bar \alpha)|  & \leq |L(\bar \theta,\alpha_k)-L(\theta_k,\alpha_k)| +|L(\theta_k,\alpha_k)-L(\bar \theta,\bar \alpha)|\\
& \leq \sup_{\alpha \in \Lambda} |L(\bar \theta,\alpha)-L(\theta_k,\alpha)|+|\sup_{\alpha \in \Lambda}L(\theta_k,\alpha)-\sup_{\alpha \in \Lambda}L(\bar \theta,\alpha)|\\
& \leq 2 \sup_{\alpha \in \Lambda} |L(\bar \theta,\alpha)-L(\theta_k,\alpha)|,
\end{align*}
which tends to zero as $k \to \infty$ by $(H'_D)$ and $(H'_p)$. Therefore, $L(\bar \theta,z)=L(\bar \theta,\bar \alpha)$ and, in turn, $z=\bar \alpha$ by $(H_1)$. Since $z \in \partial \Delta$ and $\bar \alpha \in \Delta^{\circ}$, this is a contradiction.
\end{proof}
\begin{lem}
\label{lem-tech1}Under Assumptions $(H'_{\emph{reg}})$, $(H_1)$, and $(H_{\emph{loc}})$, one has $\hat I_3 \to HV(\bar \theta)$ almost surely.
\end{lem}
\begin{proof}
We have
$$
\hat I_3=\int_0^1 H\hat V(\hat \theta+t(\hat \theta-\bar \theta)){\rm d}t=\int_0^1 \big(H_1\hat L(\hat \theta_t,\hat \alpha(\hat \theta_t))+J(\nabla_1\hat L(\hat \theta_t,\cdot))_{\hat \alpha(\hat\theta_t)}J(\hat \alpha)_{\hat \theta_t}\big){\rm d}t,
$$
where we set $\hat \theta_t =\hat \theta+t(\hat \theta-\bar \theta)$. Note that $\hat \theta_t \in U$ for all $t\in [0,1]$. By Lemma \ref{differentials},
\begin{align*}
&\sup_{t \in [0,1]}\|H_1\hat L({\hat \theta}_t,\hat \alpha({\hat \theta}_t))-H_1L({\hat \theta}_t,\hat \alpha(\hat \theta_t))\|\\
&\quad \leq \sup_{\theta \in \Theta, \alpha \in \Lambda}\|H_1\hat L(\theta,\alpha)-H_1L(\theta,\alpha)\|\to 0\quad \mbox{almost surely}.
\end{align*}
Also, by Theorem \ref{consistency}, for all $t\in [0,1]$, $\hat \theta_t \to \bar \theta$ almost surely. Besides,
\begin{align*}
|L(\bar \theta,\hat \alpha(\hat \theta_t))-L(\bar \theta,\alpha(\bar \theta))|& \leq |L(\bar \theta,\hat \alpha(\hat \theta_t))-L(\hat \theta_t,\hat \alpha(\hat \theta_t))|+|L(\hat \theta_t,\hat \alpha(\hat \theta_t))-L(\bar \theta,\alpha(\bar \theta))|\\
& \leq \sup_{\alpha \in \Lambda} |L(\bar \theta,\alpha)-L(\hat \theta_t,\alpha)|+2\sup_{\theta \in \Theta,\alpha \in \Lambda}|\hat L(\theta,\alpha)-L(\theta,\alpha)|.
\end{align*}
Thus, via  $(H'_{\text{reg}})$, $(H_1)$, and Lemma \ref{differentials}, we conclude that almost surely, for all $t \in [0,1]$, $\hat \alpha (\hat \theta_t) \to \alpha(\bar \theta)=\bar \alpha$. Accordingly, almost surely, for all $t \in [0,1]$, $H_1L(\hat \theta_t,\hat \alpha(\hat \theta_t))\to H_1L(\bar \theta,\bar \alpha)$. Since $H_1L(\theta,\alpha)$ is bounded under $(H'_D)$ and $(H'_p)$, the Lebesgue dominated convergence theorem leads to
\begin{equation}
\label{SK1}
\int_0^1 H_1 \hat L(\hat \theta_t,\hat \alpha(\hat \theta_t)){\rm d}t \to H_1L(\bar \theta,\bar \alpha) \quad \mbox{almost surely}.
\end{equation}
Furthermore,
$$J(\hat \alpha)_{\theta}=-H_2\hat L(\theta,\hat \alpha(\theta))^{-1}J\big(\nabla_2\hat L(\cdot,\hat \alpha(\theta))\big)_{\theta}, \quad \forall (\theta,\alpha) \in U \times \alpha(U),$$
where $U$ is the open set defined in the proof of Theorem \ref{clt}. By the cofactor method, $H_2\hat L(\theta,\alpha)^{-1}$ takes the form
$$H_2\hat L(\theta,\alpha)^{-1}=\frac{\hat c(\theta,\alpha)}{\mbox{det}(H_2\hat L(\theta, \alpha))},$$
where $\hat c(\theta,\alpha)$ is the matrix of cofactors associated with $H_2\hat L(\theta,\alpha)$. Thus, each component of $-H_2\hat L(\theta,\alpha)^{-1}J(\nabla_2\hat L(\cdot,\alpha))_{\theta}$ is a quotient of a multilinear form of the partial derivatives of $\hat L$ evaluated at $(\theta,\alpha)$ divided by $\mbox{det}(H_2\hat L(\theta,\alpha))$, which is itself a multilinear form in the $\frac{\partial^2 \hat L}{\partial \alpha_i \partial \alpha_j}(\theta,\alpha)$. Hence, by Lemma \ref{differentials}, we have
$$\sup_{\theta \in U,\alpha \in \alpha(U)}\|H_2\hat L(\theta,\alpha)^{-1} J(\nabla_2\hat L(\cdot,\alpha))_{\theta}-H_2L(\theta,\alpha)^{-1}J(\nabla_2L(\cdot,\alpha))_{\theta}\|\to 0\,\, \mbox{ almost surely}.$$
So, for all $n$ large enough,
\begin{align*}
&\sup_{t \in [0,1]}\|J(\hat \alpha)_{\hat \theta_t}+H_2L(\hat \theta_t,\hat \alpha(\hat \theta_t))^{-1}J\big(\nabla_2L(\cdot,\hat \alpha(\hat \theta_t))\big)_{\hat \theta_t}\|\\
& \quad \leq \sup_{\theta \in U,\alpha \in \alpha(U)}\|H_2\hat L(\theta,\alpha)^{-1}J(\nabla_2 \hat L(\cdot,\alpha))_{\theta}-H_2L(\theta,\alpha)^{-1}J(\nabla_2L(\cdot,\alpha))_{\theta}\|\\
& \quad \to 0\quad \mbox{almost surely}.
\end{align*}
We know that almost surely, for all $t\in [0,1]$, $\hat \alpha(\hat \theta_t)\to \bar \alpha$. Thus, since the function $U\times \alpha(U) \ni (\theta,\alpha)\mapsto H_2 L(\theta,\alpha)^{-1}J(\nabla_2 L(\cdot,\alpha))_{\theta}$ is continuous, we have almost surely, for all $t\in [0,1]$,
$$H_2 \hat L(\hat \theta_t,\hat \alpha(\hat \theta_t))^{-1}J\big(\nabla_2 \hat L(\cdot,\hat \alpha(\hat \theta_t))\big)_{\hat \theta_t}\to H_2L(\bar \theta,\bar \alpha)^{-1}J(\nabla_2 L(\cdot,\bar \alpha))_{\bar \theta}.$$
Therefore, almost surely, for all $t\in [0,1]$, $J(\hat \alpha)_{\hat \theta_t} \to J(\alpha)_{\bar \theta}$. Similarly,  almost surely, for all $t\in [0,1]$, $J(\nabla_1 \hat L(\hat \theta_t,\cdot))_{\hat \alpha(\hat \theta_t)}\to J(\nabla_1L(\bar \theta,\cdot))_{\bar \alpha}$. All involved quantities are uniformly bounded in $t$, and so, by the Lebesgue dominated convergence theorem, we conclude that
\begin{equation}
\label{SK2}
\int_0^1 J(\nabla_1\hat L(\hat \theta_t,\cdot))_{\hat \alpha(\hat \theta_t)}J(\hat \alpha)_{\hat \theta_t}{\rm d}t\to J(\nabla_1L(\bar \theta,\cdot))_{\bar \alpha}J(\alpha)_{\bar \theta}\quad \mbox{almost surely}.
\end{equation}
Consequently, by combining (\ref{SK1}) and (\ref{SK2}),
$$\hat I_3 \to H_1L(\bar \theta,\bar \alpha)+J(\nabla_1L(\bar \theta,\cdot))_{\bar \alpha}J(\alpha)_{\bar \theta}=HV(\bar \theta) \quad \mbox{almost surely},$$
as desired.
\end{proof}
\subsubsection*{Acknowledgments} We thank Flavian Vasile (Criteo) for stimulating discussions 
and insightful suggestions.
\bibliography{biblio-adversarial}

\begin{thebibliography}{14}
\providecommand{\natexlab}[1]{#1}
\providecommand{\url}[1]{\texttt{#1}}
\expandafter\ifx\csname urlstyle\endcsname\relax
  \providecommand{\doi}[1]{doi: #1}\else
  \providecommand{\doi}{doi: \begingroup \urlstyle{rm}\Url}\fi

\bibitem[Angles and Mallat(2018)]{AnMa18}
T.~Angles and S.~Mallat.
\newblock Generative networks as inverse problems with scattering transforms.
\newblock In \emph{International Conference on Learning Representations}, 2018.

\bibitem[Arjovsky and Bottou(2017)]{ArBo17}
M.~Arjovsky and L.~Bottou.
\newblock Towards principled methods for training generative adversarial
  networks.
\newblock In \emph{International Conference on Learning Representations}, 2017.

\bibitem[Arjovsky et~al.(2017)Arjovsky, Chintala, and Bottou]{ArChBo17}
M.~Arjovsky, S.~Chintala, and L.~Bottou.
\newblock {W}asserstein generative adversarial networks.
\newblock In D.~Precup and Y.~Whye Teh, editors, \emph{Proceedings of the 34th
  International Conference on Machine Learning}, volume~70, pages 214--223.
  Proceedings of Machine Learning Research, 2017.

\bibitem[Devroye(1997)]{De97}
L.~Devroye.
\newblock Universal smoothing factor selection in density estimation: {T}heory
  and practice.
\newblock \emph{TEST}, 6:\penalty0 223--320, 1997.

\bibitem[Dziugaite et~al.(2015)Dziugaite, Roy, and Ghahramani]{DzRoGh15}
G.K. Dziugaite, D.M. Roy, and Z.~Ghahramani.
\newblock Training generative neural networks via maximum mean discrepancy
  optimization.
\newblock In \emph{Proceedings of the Thirty-First Conference on Uncertainty in
  Artificial Intelligence}, pages 258--267. AUAI Press, Arlington, 2015.

\bibitem[Endres and Schindelin(2003)]{EnSc03}
D.M. Endres and J.E. Schindelin.
\newblock A new metric for probability distributions.
\newblock \emph{IEEE Transactions on Information Theory}, 49:\penalty0
  1858--1860, 2003.

\bibitem[Goodfellow(2016)]{Go16}
I.~Goodfellow.
\newblock \emph{NIPS 2016 Tutorial: Generative Adversarial Networks}.
\newblock arXiv:1701.00160, 2016.

\bibitem[Goodfellow et~al.(2014)Goodfellow, Pouget-Abadie, Mirza, Xu,
  Warde-Farley, Ozair, Courville, and Bengio]{GoPoMiXuWaOzCoBe14}
I.J. Goodfellow, J.~Pouget-Abadie, M.~Mirza, B.~Xu, D.~Warde-Farley, S.~Ozair,
  A.~Courville, and J.~Bengio.
\newblock Generative adversarial nets.
\newblock In Z.~Ghahramani, M.~Welling, C.~Cortes, N.D. Lawrence, and K.Q.
  Weinberger, editors, \emph{Advances in Neural Information Processing Systems
  27}, pages 2672--2680. Curran Associates, Inc., Red Hook, 2014.

\bibitem[Liu et~al.(2017)Liu, Bousquet, and Chaudhuri]{LiBoCh07}
S.~Liu, O.~Bousquet, and K.~Chaudhuri.
\newblock Approximation and convergence properties of generative adversarial
  learning.
\newblock In I.~Guyon, U.V. Luxburg, S.~Bengio, H.~Wallach, R.~Fergus,
  S.~Vishwanathan, and R.~Garnett, editors, \emph{Advances in Neural
  Information Processing Systems 30}, pages 5551--5559. Curran Associates,
  Inc., Red Hook, 2017.

\bibitem[McDiarmid(1989)]{Mc89}
C.~McDiarmid.
\newblock On the method of bounded differences.
\newblock In J.~Siemons, editor, \emph{Surveys in Combinatorics}, London
  Mathematical Society Lecture Note Series 141, pages 148--188. Cambridge
  University Press, Cambridge, 1989.

\bibitem[Nowozin et~al.(2016)Nowozin, Cseke, and Tomioka]{NoCsTo16}
S.~Nowozin, B.~Cseke, and R.~Tomioka.
\newblock f-{GAN}: {T}raining generative neural samplers using variational
  divergence minimization.
\newblock In D.D. Lee, M.~Sugiyama, U.V. Luxburg, I.~Guyon, and R.~Garnett,
  editors, \emph{Advances in Neural Information Processing Systems 29}, pages
  271--279. Curran Associates, Inc., Red Hook, 2016.

\bibitem[Salimans et~al.(2016)Salimans, Goodfellow, Zaremba, Cheung, Radford,
  and Chen]{SaGoZaChRaCg16}
T.~Salimans, I.~Goodfellow, W.~Zaremba, V.~Cheung, A.~Radford, and X.~Chen.
\newblock Improved techniques for training {GAN}s.
\newblock In D.D. Lee, M.~Sugiyama, U.V. Luxburg, I.~Guyon, and R.~Garnett,
  editors, \emph{Advances in Neural Information Processing Systems 29}, pages
  2234--2242. Curran Associates, Inc., Red Hook, 2016.

\bibitem[van Handel(2016)]{Ha16}
R.~van Handel.
\newblock \emph{Probability in High Dimension}.
\newblock APC 550 Lecture Notes, Princeton University, 2016.

\bibitem[Zhang et~al.(2018)Zhang, Liu, Zhou, Xu, and He]{ZhLiZhXuHe18}
P.~Zhang, Q.~Liu, D.~Zhou, T.~Xu, and X.~He.
\newblock On the discriminative-generalization tradeoff in {GAN}s.
\newblock In \emph{International Conference on Learning Representations}, 2018.

\end{thebibliography}

\end{document}